\documentclass[letterpaper]{article} % DO NOT CHANGE THIS
\usepackage{aaai24}  % DO NOT CHANGE THIS
\usepackage{times}  % DO NOT CHANGE THIS
\usepackage{helvet}  % DO NOT CHANGE THIS
\usepackage{courier}  % DO NOT CHANGE THIS
\usepackage[hyphens]{url}  % DO NOT CHANGE THIS
\usepackage{graphicx} % DO NOT CHANGE THIS
\usepackage{natbib}  % DO NOT CHANGE THIS AND DO NOT ADD ANY OPTIONS TO IT
\usepackage{caption} % DO NOT CHANGE THIS AND DO NOT ADD ANY OPTIONS TO IT
\usepackage{algorithm}
\usepackage{algorithmic}

\usepackage{newfloat}
\usepackage{listings}
\DeclareCaptionStyle{ruled}{labelfont=normalfont,labelsep=colon,strut=off} % DO NOT CHANGE THIS
\floatstyle{ruled}
\newfloat{listing}{tb}{lst}{}
\floatname{listing}{Listing}
\usepackage{xcolor} % TODO: Delete
\usepackage{multicol}
\usepackage{multirow}
\usepackage{array}
\usepackage{enumerate}

\usepackage{csquotes}
\usepackage{amsmath}
\usepackage{amssymb}
\usepackage{amsthm}
\theoremstyle{definition}
\newtheorem{definition}{Definition}[section]
\newtheorem{assumption}[definition]{Assumption}
\theoremstyle{plain}
\newtheorem{theorem}[definition]{Theorem}
\newtheorem{lemma}[definition]{Lemma}
\newtheorem{corollary}[definition]{Corollary}
\theoremstyle{remark}
\newtheorem{remark}[definition]{Remark}
\allowdisplaybreaks

\newcommand{\R}{\mathbb{R}}
\newcommand{\N}{\mathcal{N}} 
\newcommand{\Pa}{\mathcal{P}} 
\newcommand{\Le}{\mathcal{L}} 
\DeclareMathOperator{\sgn}{sgn}
\DeclareMathOperator{\rk}{rk}

\title{Physics-Informed Graph Neural Networks for Water Distribution Systems}
\author{
    Inaam Ashraf,
    Janine Strotherm,
    Luca Hermes,
    Barbara Hammer
}
\affiliations{
    Center for Cognitive Interaction Technology, Bielefeld University,
    Inspiration 1, 33619 Bielefeld, Germany\\
    mashraf@techfak.uni-bielefeld.de, jstrotherm@techfak.uni-bielefeld.de, lhermes@techfak.uni-bielefeld.de, bhammer@techfak.uni-bielefeld.de
}

\begin{document}

\maketitle

\begin{abstract}
Water distribution systems (WDS) are an integral part of critical infrastructure which is pivotal to urban development. As 70\% of the world's population will likely live in urban environments in 2050, efficient simulation and planning tools for WDS play a crucial role in reaching UN's sustainable developmental goal (SDG) 6 -- \enquote{Clean water and sanitation for all}. In this realm, we propose a novel and efficient machine learning emulator, more precisely, a physics-informed deep learning (DL) model, for hydraulic state estimation in WDS. Using a recursive approach, our model only needs a few graph convolutional neural network (GCN) layers and employs an innovative  algorithm based on message passing. Unlike conventional machine learning tasks, the model uses hydraulic principles to infer two additional hydraulic state features in the process of reconstructing the available ground truth feature in an unsupervised manner. To the best of our knowledge, this is the first DL approach to emulate the popular hydraulic simulator EPANET, utilizing no additional information. Like most DL models and unlike the hydraulic simulator, our model demonstrates vastly faster emulation times that do not increase drastically with the size of the WDS. Moreover, we achieve high accuracy on the ground truth and very similar results compared to the hydraulic simulator as demonstrated through experiments on five real-world WDS datasets.  

\end{abstract}

\section{Introduction} \label{sec: introduction}

As it is expected that 70\% of world's population will live in urban areas in 2050 \cite{owidurbanization}, urban development presents a number of complex challenges. In the light of deep uncertainties caused, among others, by climate change and migration, building smart cities requires flexible and efficient short and long-term planning, monitoring and expansion of its critical infrastructure, i.e., its transportation systems, electricity distribution systems, and water distribution systems (WDS). Reliable WDS, in particular, constitute one critical demand of UN's SDG 6.
In this realm, AI technologies carry great promises since
AI can be used for intelligent planning, monitoring and control of these systems \cite{pmlr-v176-eichenberger22a,smartcities4020029,doi:10.1061/(ASCE)PS.1949-1204.0000646}. 
Currently, there exist first approaches of both classical AI tools and Deep Learning (DL) to solve various tasks in critical infrastructure systems \cite{doi:10.1061/(ASCE)IS.1943-555X.0000477}.
In this contribution, we  center on one important aspect related to planning and control of WDS: Since WDS cannot be experimented with in practice, the efficient emulation of realistic large-scale WDS constitutes a crucial prerequisite for planning and optimization based on such digital twins.
We investigate how to derive an efficient emulation of WDS based on DL.

Currently, hydraulic simulators play a key role in planning and expansion of WDS. The EPANET simulator \cite{rossman2020epanet} is one of the most popular hydraulic simulator for WDS. However, it requires considerable resources for extended time simulations that scale non-linearly with the size of the WDS, making WDS optimization a slow process. 
Since critical infrastructure systems such as WDS are naturally structured as graphs, Graph Neural Networks (GNNs) carry great promises as an alternative, faster, modeling tool.

A WDS consists of $N_n$ junctions and $N_e$ links connecting these junctions. Junctions can be reservoirs, tanks or water consumers, while links can be pipes, valves or pumps. 
The first step in planning a WDS is to define the structure, i.e., the number of junctions and links. Next, one can estimate the hydraulic state at every junction and link, given a set of features. These features are the head (see section \ref{sec: methodology}) at the reservoirs and the water demand of every consumer. Based on this, one needs to estimate  the head at \textit{every} junction and the flow through \textit{every} link. EPANET iteratively solves $N_n$ differential equations \cite{rossman2020epanet}, leading to slow simulations for extended time or larger systems.

In this paper, we combine a local graph convolutional neural network (GCN) model with a physics informed global algorithm that emulates the hydraulic simulator. Since explicit values for heads or flows are not available, we use hydraulic formulas governing the relationships between demands, heads and flows, instead, as physics-informed learning signals.

Our key contributions are as follows:
\begin{itemize}
    \item We propose a novel GNN architecture which combines 
    trainable local aspects with global state estimation.
    \item We use hydraulic principles to infer the required state features rather than example-driven supervised learning.
    \item We display the accuracy of the model for various WDS benchmarks.
    \item We show that we can significantly reduce the inference times compared to the hydraulic simulator.
    \item To the best of our knowledge, this is the first DL approach emulating the simulator using no additional information.    
\end{itemize}

\begin{figure*}[!htbp]
\centering
\resizebox{.8\textwidth}{!}{
\includegraphics[]{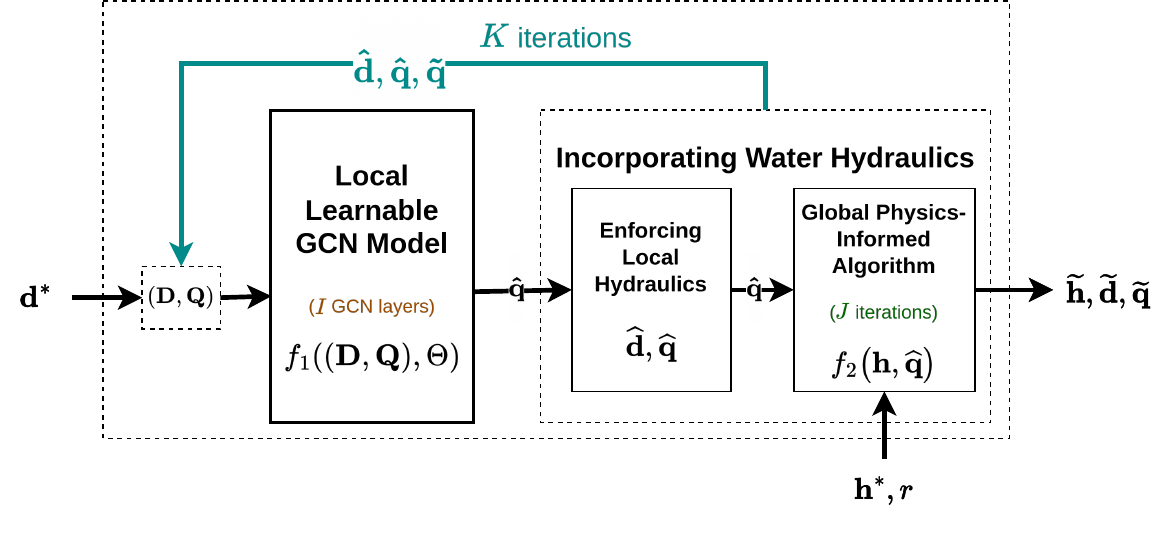} 
}
\caption{The model architecture: The local GCN model $f_1$ learns from the global physics-informed algorithm $f_2$ through multiple iterations.} 
\label{fig: architecture}
\end{figure*}

\section{Related Work} \label{sec: related work}

Curently, ML technologies have been proposed for  specific tasks in WDS such as leakage detection \cite{Fan2021}, modeling virtual sensors \cite{su15042981}, or demand prediction \cite{WU2023104545}. 
 To the best of our knowledge, the complex task of state estimation has only been dealt with by \cite{xing2022stateestimation}, addressing model hydraulics using GNNs.  More specifically, the EPANET simulator \cite{rossman2020epanet} is emulated using various information such as sparse pressure sensor measurements as features; this additional information makes the problem significantly easier to solve and it is usually not available 
in a WDS planning phase. Moreover,  the method is demonstrated on an artificial small WDS only.

In  deployed WDS, pressure readings from few sensors are available. This leads to the task of pressure estimation at all nodes in a WDS. \cite{hajgato2021pressure} used spectral GCNs to achieve encouraging results demonstrated by extensive experiments. Spectral GCNs do not fully utilize the structural information in a graph. Spatial GCNs were used by \cite{ashraf2023spatial} to significantly improve the performance. 
Generally, there exists a variety of different GNN structures,
such as recursive graph and tree models  \cite{scarselli2009,diss},
spectral GCNs \cite{Bruna2014SpectralNA,kipf2017semi,defferrard2016convolutional,henaff2015deep,levie2018cayleynets,li2018adaptive}, or spatial GCNs based on the spectral graph kernel \cite{hamilton2017inductive,monti2017geometric,gao2018large,niepert2016learning,xu2018powerful,velickovic2018graph}. Since these local operations are analogous to sending and receiving messages between nodes, spatial GCNs are also called message passing neural networks. In this contribution, we will extend such first advances towards an efficient physics-informed model capable of emulating large-scale WDS based on demands signals only.

The majority of ML training schemes  are supervised based on example data; moreover, DL often requires a considerable amount of training data for valid generalization. When dealing with complex systems, such information is often not easily available. Hence researchers investigate how to minimize the required amount of training data when using ML models as efficient surrogates for complex systems \cite{ruff2023surrogate}, and some technologies even enable the exchange of training data by general physical principles which can be incorporated into the learning scheme \cite{RAISSI2019686}. In this contribution, we introduce a novel graph architecture and learning scheme, which enables us to train a surrogate model based on physical principles only.

\newpage 

\section{Methodology}\label{sec: methodology}

WDS systems are characterized by demands $d \in \R_+$, flows $q \in \R$ and heads $h \in \R_+$. The latter is a measure of energy with the relationship $h = p + \epsilon$ to the pressure $p \in \R_+$ and the elevation of the junction $\epsilon \in \R_+$ \cite{rossman2020epanet}. Before we further explain the water hydraulics in section \ref{subsection_WaterHydraulics}, we present an overview of our methodology.

\subsection{Overview} 

We propose a spatial GCN-based and physics-informed iterative model to emulate the hydraulic simulator EPANET for WDS. As in case of EPANET, our model takes reservoir heads and consumer demands at every junction as inputs and estimates heads at every junction and flows at every link. 
The strength of our model comes from two main components.

The first component is a GCN-based, learnable function $f_1$, that makes use of the graph structure of a WDS to update flows based on demands and initial flows.
When modelling a WDS as a graph with junctions $V = \{v_1,\dots,v_{N_n}\}$ as nodes and links $E=\{ e_{vu} \; | \; \forall \, v \in V; u \in \mathcal{N}(v) \} = \{e_1,\dots,e_{N_e}\}$ as edges,
 $f_1$ takes \textit{demand node features} $\mathbf{D} \in \R^{N_n \times M_n}$ and \textit{flow edge features} $\mathbf{Q} \in \R^{N_e \times M_e}$ as inputs with $M_n = M_e = 2$ as feature dimension per node and edge.  $f_1$ estimates updated \textit{flows} $\mathbf{\hat{q}} = (\hat{q}_e)_{e \in E} \in \R^{N_e}$:
\begin{align}
\label{align_firstComponent}
\begin{split}
    f_1(\cdot, \Theta): ~ \R^{N_n \times M_n} \times \R^{N_e \times M_e} &\longrightarrow \R^{N_e} \\
    (\mathbf{D},\mathbf{Q}) &\longmapsto \mathbf{\hat{q}},
\end{split}
\end{align}
\noindent where the matrices can be rewritten as
\begin{align*}
    (\mathbf{D},\mathbf{Q}) = ((\mathbf{d}_1, \mathbf{d}_2), (\mathbf{q}_1, \mathbf{q}_2)) = ((\mathbf{d}_v^T)_{v \in V}, (\mathbf{q}_e^T)_{e \in E})
\end{align*}

\noindent with  $~\mathbf{d}_v^T = (d_{v1}, d_{v2})$ and $\mathbf{q}_e^T = (q_{e1}, q_{e2})$
consisting of the ground truth demand $d_{v1}$, as well as firstly initialized and afterwards to be updated demand $d_{v2}$ and flows $q_{e1}$ and $q_{e2}$ for each node $v \in V$ and edge $e \in E$. We provide further details on initialization in section \ref{subsection_OverallModelandTraining}.  
Using the updated flows $\mathbf{\hat{q}}$, we also compute associated updated demands $\mathbf{\hat{d}} \in \R^{N_n}$ (cf. eq. \eqref{align_MassBalance}).

The second component is a physics-informed function 
\begin{align}
\label{align_secondComponent}
\begin{split}
    f_2: ~ \R^{N_n} \times \R^{N_e} & \longrightarrow \R^{N_n} \times \R^{N_e} \\
    (\mathbf{h}, \mathbf{\hat{q}}) & \longmapsto (\mathbf{\tilde{h}}, \mathbf{\tilde{d}}, \mathbf{\tilde{q}})
\end{split}    
\end{align}
\noindent that makes use of initialized \textit{heads}  $\mathbf{h} \in \R^{N_n}$ and the predicted flows $\mathbf{\hat{q}} = f_1(\mathbf{D},\mathbf{Q}) \in \R^{N_e}$ to compute updated heads $\mathbf{\tilde{h}} \in \R^{N_n}$, flows $\mathbf{\tilde{q}} \in \R^{N_e}$ and demands $\mathbf{\tilde{d}} \in \R^{N_n}$ based on a recursive scheme that takes into account water hydraulics. 

It becomes obvious that during computation of $f_1$ and $f_2$, we consider several demands as well as flows.
As we do not have labels for this task, these estimations replace prediction and label pairs: They might differ in the beginning, but converge towards the same values during training of the function $f := f_2 \circ f_1(\cdot, \Theta)$, relying on the several parameters $\Theta$ of the GCN layers (cf. theorem \ref{theorem_ConvergenceOfAlgorithm}, \ref{theorem_MotivationAlgorithm} and section \ref{subsection_OverallModelandTraining}).

A visualization of the model architecture is given in fig. \ref{fig: architecture}. In the following sections, we introduce the components and the functionality of the overall resulting model in detail.

\subsection{Local Learnable GCN-Model}
\label{subsection_LearnableGCNModel}

As a first step of the computation of $f_1$, the node and edge features are embedded in a latent space with dimension $M_l$ by applying fully connected neural network layers $\alpha$ and $\beta$:
\begin{align*}
    \R^{N_n \times M_n} \times \R^{N_e \times M_e} 
    & \longrightarrow
    \R^{N_n \times M_l} \times \R^{N_e \times M_l} \\
    ((\mathbf{d}_v^T)_{v \in V}, (\mathbf{q}_e^T)_{e \in E}) 
    & \longmapsto
    ((\mathbf{g}_v^T)_{v \in V}, (\mathbf{z}_e^T)_{e \in E}),
\end{align*}

\noindent where for $v \in V$ and $e \in E$, respectively, we define
\begin{align*}
    \mathbf{g}_v := \alpha(\mathrm{SeLU}(\mathbf{d}_v))
    \text{ and } 
    \mathbf{z}_e := \beta(\mathrm{SeLU}(\mathbf{q}_e)).
\end{align*}

\noindent Hereby, the SeLU activation function is used for inducing self-normalizing properties  \cite{Gunter2017SNNs}.

Afterwards, we conduct the standard three-step process of message generation, message aggregation and feature update \cite{li2020GENConv}. This process is repeated for $I$ 
GCN layers, starting with $\mathbf{g}_v^{(0)} = \mathbf{g}_v$ and $\mathbf{z}_e^{(0)} = \mathbf{z}_e$ for all $v \in V$ and $e \in E$, respectively:
For the $i$-th layer for $ i = 0, ..., I-1$, for \textbf{message generation}, the latent node and edge features are fed to a Multi-Layer Perceptron (MLP) $\gamma^{(i)}$:
\begin{align*}
    \R^{N_n \times M_l} \times \R^{N_e \times M_l}
    & \longrightarrow
    \R^{N_e \times M_l} \\
    ((\mathbf{g}_v^{(i)})^T)_{v \in V}, ((\mathbf{z}_e^{(i)})^T)_{e \in E})
    & \longmapsto
    ((\mathbf{m}_e^{(i)})^T)_{e \in E},
\end{align*}

\noindent where for $e = e_{vu} \in E$, we define
\begin{align*}
    \mathbf{m}_e^{(i)}
    := 
    \gamma^{(i)}(\mathrm{SeLU}(\mathbf{g}_u^{(i)} \parallel \mathbf{g}_v^{(i)} \parallel \mathbf{z}_e^{(i)}))  
    \text{ with } 
    u \in \mathcal{N}(v)
\end{align*}

\noindent and where $\cdot \parallel \cdot$ denotes vector concatenation. 
For \textbf{message aggregation}, max aggregation is used:
\begin{align*}
    \R^{N_e \times M_l}
    & \longrightarrow
    \R^{N_n \times M_l} \times \R^{N_e \times M_l}\\
    ((\mathbf{m}_e^{(i)})^T)_{e \in E}
    & \longmapsto
    (((\mathbf{m}_v^{(i)})^T)_{v \in V}, ((\mathbf{m}_e^{(i)})^T)_{e \in E}),
\end{align*}

\noindent  where for $v \in V$, we define
\begin{align*}
    \mathbf{m}_v^{(i)} := \max_{u \in \mathcal{N}(v)} \mathbf{m}_{e_{vu}}^{(i)}.
\end{align*}

\noindent For the \textbf{feature update}, the node messages are fed to another MLP $\eta^{(i)}$ whereas the edge features are updated as the edge messages:
{\scriptsize 
\begin{align*}
    \R^{N_n \times M_l} \times \R^{N_e \times M_l}
    & \rightarrow 
    \R^{N_n \times M_l} \times \R^{N_e \times M_l}\\
    (((\mathbf{m}_v^{(i)})^T)_{v \in V}, ((\mathbf{m}_e^{(i)})^T)_{e \in E})
    & \mapsto
    ((\mathbf{g}_v^{(i+1)})^T)_{v \in V}, ((\mathbf{z}_e^{(i+1)})^T)_{e \in E}),
\end{align*}
}%
\noindent where for $e = e_{vu} \in E$, we define
\begin{align*}
    \mathbf{g}_v^{(i+1)} := \eta^{(i)}(\mathbf{m}_v^{(i)}) 
    \text{ and } 
    \mathbf{z}_e^{(i+1)} := \mathbf{m}_e^{(i)}.
\end{align*}

\noindent After the last GCN layer $I$, new flows are estimated using another MLP $\lambda$:
\begin{align*}
    \R^{N_n \times M_l} \times \R^{N_e \times M_l}
    & \longrightarrow
    \R^{N_e}\\
    ((\mathbf{g}_v^{(I)})^T)_{v \in V}, ((\mathbf{z}_e^{(I)})^T)_{e \in E})
    & \longmapsto
    (\hat{q}_e)_{e \in E},
\end{align*}

\noindent where we define with slight abuse of notation\footnote{
    Note that the dependence of the first component $\mathbf{q}_1 \in \R^{N_e}$ of the input flows $\mathbf{Q} \in \R^{N_e \times M_e}$ added to the output of the MLP $\lambda$ can easily be integrated in all of the previous functions by adding a second and third factor to each of the tuples and triples, respectively, and using the identity on $\R^{N_e}$ to make $\mathbf{q}_1$ as an input of the latter function. We have only waived this step to save space.
}
\begin{align*}
    \hat{q}_e 
    := 
    q_{e1} 
    +
    \lambda(\mathrm{SeLU}(\mathbf{g}_u^{(I)} \parallel \mathbf{g}_v^{(I)} \parallel \mathbf{z}_e^{(I)})
    \text{ with } 
    u \in \mathcal{N}(v)
\end{align*}

\noindent for $e = e_{vu}  \in E$. 
As $f_1((\mathbf{D}, \mathbf{Q}), \Theta) := \mathbf{\hat{q}} = (\hat{q}_e)_{e \in E}$, this is the last step of the learnable GCN-based model $f_1$, determined by all parameters $\Theta$ of all networks $\alpha, \beta, \gamma^{(i)}, \eta^{(i)}$ and $\lambda$ and for all $i = 0,...,I-1$. While the structure of these networks is standard, their dynamics are special:
We iteratively update the outputs of the overall model $f_2 \circ f_1$ before updating the parameters $\Theta$, as we will discuss in section \ref{subsection_OverallModelandTraining}. 

\subsection{Water Hydraulics} 
\label{subsection_WaterHydraulics}

Up to this point, we use GCNs to estimate flows $f_1((\mathbf{D}, \mathbf{Q}), \Theta) = \mathbf{\hat{q}}$ based on input demands $\mathbf{D}$ and input flows $\mathbf{Q}$. Next, we explain how we change these flows such that they obey the principles of water hydraulics.

\begin{figure}[tb]
\centering
\resizebox{.6\columnwidth}{!}{
\includegraphics[]{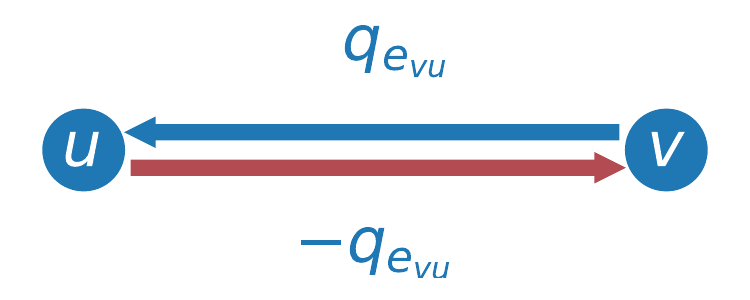} 
}
\caption{The flow $q_{e_{uv}}$ from node $u$ to $v$ has to be equal to the negative of the flow $q_{e_{vu}}$ from $v$ to $u$.} 
\label{fig: flow_direction_relationship}
\end{figure}

\begin{figure}[tb]
\centering
\resizebox{.6\columnwidth}{!}{
\includegraphics[]{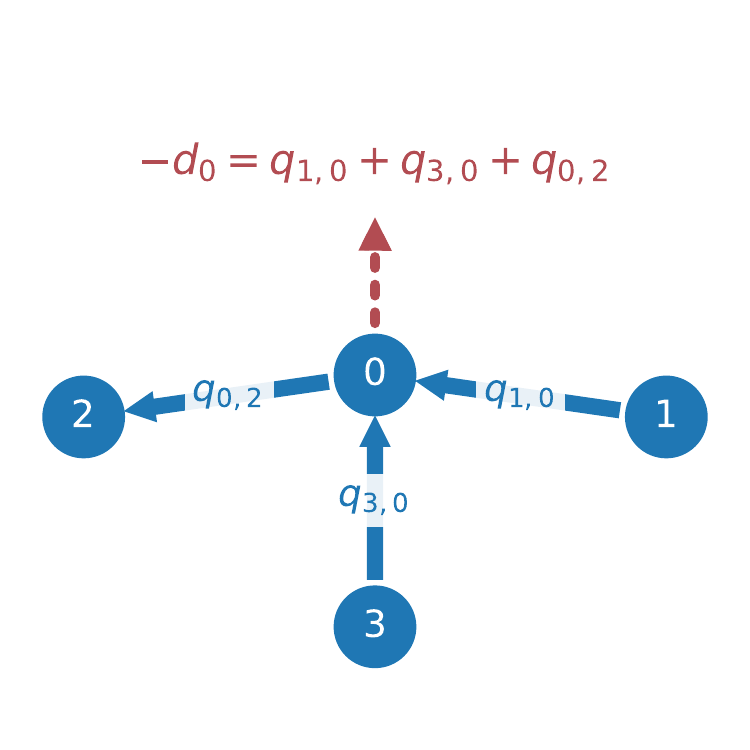} 
}
\caption{The sum of inflows $q_{1,0}, q_{3,0}$ and outflow $q_{0,2}$ must be equal to the negative of the demand $d_{0}$ at node $0$.} 
\label{fig: flow_demand_relationship}
\end{figure}

\begin{figure}[!htbp]
\centering
\resizebox{.6\columnwidth}{!}{
\includegraphics[]{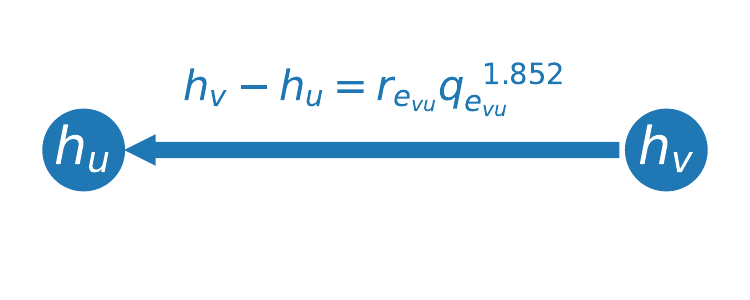} 
}
\caption{The relationship between pressure heads at neighbouring nodes and the flow in connecting pipe.} 
\label{fig: pipe_headloss_relationship}
\end{figure}

On the one hand, for two neighboring nodes $v,u \in V$, the flow $q_{e_{uv}}$ from node $u$ to $v$ has to be equal to the negative of the flow $q_{e_{vu}}$ from $v$ to $u$ (fig. \ref{fig: flow_direction_relationship}). Neural networks are not inherently aware of this property. Therefore, after computing flows in both directions by using bidirectional edges in our graph, we discard half of the flows in $\mathbf{\hat{q}}$ and replace each of those with the negative of the remaining corresponding flow.\footnote{
Since at the beginning, we do not know flow directions, edge directions can be initialized arbitrarily and any of the two half sets of flows can be used.}

On the other hand, \textit{true} demands $\mathbf{d}^* = (d_v^*)_{v \in V}$, \textit{true} heads $\mathbf{h}^* = (h_v^*)_{v \in V}$ and \textit{true} flows $\mathbf{q}^* = (q_e^*)_{e \in E}$ are subject to further hydraulic principles: Firstly, due to flow preservation, the sum of incoming and outcoming flows $e_{vu}$ between a node $v \in V$ and its neighbors $u \in \mathcal{N}(v)$ is related to its demand $d_v$ by
\begin{align}
\label{align_MassBalance}
    \textstyle 
    \sum_{u \in \mathcal{N}(v)} q_{e_{vu}}^* = - d_v^*.
\end{align}
\noindent Hereby, we use the convention that the flow $q_{e_{vu}}$ from $v$ to $u$ has a positive sign if it corresponds to an \textit{outflow} (including demands) and a negative sign if it corresponds to an \textit{inflow} \cite{rossman2020epanet}. 
This can be seen by the second property, which states that for any two neighboring nodes $v,u \in V$, pressure heads are related to flows by
\begin{align}
\label{align_HeadLoss}
    h_v^* - h_u^* = r_{e_{vu}} \sgn(q_{e_{vu}}^*) |q_{e_{vu}}^*|^x,
\end{align}

\noindent where $x = 1.852$ and
\begin{align}
\label{align_ConstantR}
    r_{e_{vu}} = 
    10.667 ~ l_{e_{vu}} 
    d_{e_{vu}}^{-4.871} c_{e_{vu}}^{-1.852}
    > 0
\end{align}

\noindent is a link-dependent constant based on its length $l_{e_{vu}}$, diameter $d_{e_{vu}}$ and roughness coefficient $c_{e_{vu}}$ \cite{rossman2020epanet}. Now, as water always (out)flows from a higher head to a lower head, if the head loss $h_v - h_u$ is positive as the head $h_v$ at node $v$ is higher than the head $h_u$ at the neighbor node $u$, the water outflow $q_{vu}$ needs to be positive and similarly, an inflow $q_{vu}$ needs to be negative.

As we have not utilized any head values in the model $f_1$ explicitly, during learning, 
it could choose any set of flows that can satisfy eq.\ \eqref{align_MassBalance} at every node. 
Yet, such a solution will most likely not obey eq.\ \eqref{align_HeadLoss} due to the following lemma:

\begin{lemma}
    Given $N_n$ nodes $v \in V$ and $N_e > N_n$ edges $e \in E$, there is no unique solution to eq. \eqref{align_MassBalance}. 
\end{lemma}

\begin{proof}
    The proof is given in Appendix \ref{subsection_ProofLemma3.1}.  
\end{proof}

\noindent As in practice, $N_e > N_n$ is usually the case (even for the directed graph),
we leverage the message passing framework using water hydraulics within the second component $f_2$ of our overall model $f = f_2 \circ f_1$ to guide the model to a plausible solution, as described in the following section. 

\subsection{Global Physics-Informed Algorithm}
\label{subsection_GlobalPhysicsInformedAlgorithm}

The following components of $f_2$ use eq.\ \eqref{align_HeadLoss} to construct the heads $\mathbf{\tilde{h}} \in \R^{N_n}$ and the updated flows $\mathbf{\tilde{q}} \in \R^{N_e}$ as well as the updated demands $\mathbf{\tilde{d}} \in \R^{N_e}$ based on some initialized heads $\mathbf{h} \in \R^{N_n}$ and the outcome $f_1((\mathbf{D}, \mathbf{Q}), \Theta) := \mathbf{\hat{q}} \in \R^{N_e}$ of the first component. 

The computation of $f_2(\mathbf{h}, \mathbf{\hat{q}})$ is inspired by the three-step process of message generation, message aggregation and feature update, however, without the usage of trainable parameters. Instead, we define a recursive algorithm that converges to a fix point, as we will show in Theorem \ref{theorem_ConvergenceOfAlgorithm}.
Thus, starting with $\tilde{h}_v^{(0)} = h_v$ for all $v \in V$,  we repeat the following steps $J$ times where $J  \in \mathbb{N}$ is the iteration where the algorithm reaches its fixed point: For the $j$-th iteration for $j=0,...,J-1$, for \textbf{message generation}, we compute
\begin{align*}
    \R^{N_n} \times \R^{N_e}
    & \longrightarrow
    \R^{N_n} \times \R^{N_e} \\
    ((\tilde{h}_v^{(j)})_{v \in V}, (\hat{q}_e)_{e \in E})
    & \longmapsto
    ((\tilde{h}_v^{(j)})_{v \in V}, (m_e^{(j)})_{e \in E}),
\end{align*}

\noindent where for $e = e_{vu} \in E$,  we define
\begin{align} 
\label{align_EdgeMessageGeneration_PhysicsInformed}
    {m}_e^{(j)}
    := 
    {\tilde{h}}_u^{(j)}
    -
    \mathrm{ReLU}\left(- r_e  \sgn({\hat{q}}_e) |{\hat{q}}_e|^x \right).
\end{align}

\noindent For \textbf{message aggregation}, max aggregation is used:
\begin{align*}
    \R^{N_n} \times \R^{N_e}
    & \longrightarrow
    \R^{N_n} \times \R^{N_n} \\
    (({\tilde{h}}_v^{(j)})_{v \in V}, ({m}_e^{(j)})_{e \in E})
    & \longmapsto
    (({\tilde{h}}_v^{(j)})_{v \in V}, ({m}_v^{(j)})_{v \in V}),
\end{align*}

\noindent  where for $v \in V$,  we define
\begin{align*}
    {m}_v^{(j)} 
    := 
    \max_{u \in \N(v)} {m}_{e_{vu}}^{(j)}.
\end{align*}

\noindent Finally, the \textbf{feature update} outputs the heads ${\tilde{h}}^{(j+1)} \in \R^{N_n}$ of the next iteration\footnote{
    We can of course make $\mathbf{\hat{q}}$ the output of the latter function by again adding a second and third factor to each of the tuples and triples, respectively, and using the identity on $\R^{N_e}$.
}:
\begin{align*}
    \R^{N_n} \times \R^{N_n}
    & \longrightarrow
    \R^{N_n} \\
    (({\tilde{h}}_v^{(j)})_{v \in V}, ({m}_v^{(j)})_{v \in V})
    & \longmapsto
    ({\tilde{h}}_v^{(j+1)})_{v \in V},
\end{align*}

\noindent  where for $v \in V$,  we define
\begin{align}
\label{align_Algorithm}
    {\tilde{h}}_v^{(j+1)} 
    := 
    \max\{ {\tilde{h}}_v^{(j)}, {m}_v^{(j)} \}.
\end{align}

\begin{theorem}
\label{theorem_ConvergenceOfAlgorithm}
    Let $V_r \subset V$ be a subset of non-connecting nodes of the WDS with associated values $(h_v^*)_{v \in V_r}$, $L_{\max}$ the length of the longest path from any instance $v_r \in V_r$ to any instance $v \in V$ and define $c := \min_{v_r \in V_r} h_{v_r}^* - L_{\max} \cdot \max_{e \in E} \mathrm{ReLU}\left(- r_e  \sgn({\hat{q}}_e) |{\hat{q}}_e|^x \right)$.
    If we initialize $\mathbf{\tilde{h}}^{(0)} = \mathbf{h} \in \R^{N_n}$ according to
    \begin{align*}
        h_v
        :=
        \begin{cases}
            h_v^* & \text{ if } v \in V_r \\
            c & \text{ if } v \in V \setminus V_r,
        \end{cases}
    \end{align*}
        
    \noindent the recursive algorithm defined by eq.\ \eqref{align_Algorithm} converges in at most $J = L_{\max}$ steps.
\end{theorem}

\begin{proof}
    The proof is given in Appendix \ref{subsection_ProofTheorem3.2}.
\end{proof}

\noindent After this recursive algorithm has converged, the updated demands $\mathbf{\tilde{d}} \in \R^{N_n}$ as well as updated flows $\mathbf{\tilde{q}} \in \R^{N_e} $ are reconstructed from the heads $\mathbf{\tilde{h}} := \mathbf{\tilde{h}}^{(J)} \in \R^{N_n}$ based on the water hydraulics from eq. \eqref{align_MassBalance} and \eqref{align_HeadLoss}:
\begin{align*}
    \R^{N_n}
    & \longrightarrow
    \R^{N_n} \\
    ({\tilde{h}}_v)_{v \in V}
    & \longmapsto
    (({\tilde{d}}_v)_{v \in V}, ({\tilde{q}}_e)_{e \in E}),
\end{align*}

\noindent  where for $v \in V$ and $e = e_{vu} \in E$, we define
\begin{align}
\label{align_ComputeFlowsAndDemands}
\begin{split}
    {\tilde{q}}_e
    &:= 
    \sgn({\tilde{h}_v} - {\tilde{h}}_u) \cdot 
    (r_e^{-1} |{\tilde{h}_v} - {\tilde{h}}_u|)^{1/x} + \zeta, 
    \\
    {\tilde{d}}_v 
    &:= - 
    \textstyle 
    \sum_{u \in \mathcal{N}(v)} {\tilde{q}}_{e_{vu}},
\end{split}
\end{align}

\noindent and $\zeta$ is a  small number that ensures that this equation remains differentiable. This is the last step of the physics-informed component $f_2$, mapping initialized heads $\mathbf{h}$ and the flows $\mathbf{\hat{q}} = f_1((\mathbf{D}, \mathbf{Q}), \Theta)$ given by the first component $f_1$ to updated heads, demands and flows $(\mathbf{\tilde{h}}, \mathbf{\tilde{d}}, \mathbf{\tilde{q}}) = f_2((\mathbf{h}, \mathbf{\hat{q}})) = f_2((\mathbf{h}, f_1((\mathbf{D}, \mathbf{Q}), \Theta))$.
Note that all computations within this component do not involve any learnable parameters, but are only based on the hydraulics of the WDS. 

A deeper intuition of this component can be found in the ArXiv version. One important property of the physics-informed algorithm is that if the flows $\mathbf{\hat{q}}$ estimated by the GCN are correct, they remain unchanged by the algorithm:  

\begin{theorem}
\label{theorem_MotivationAlgorithm}
    If in the setting theorem \ref{theorem_ConvergenceOfAlgorithm}, $V_r$ corresponds to the reservoirs with known heads $(h_v^*)_{v \in V_r}$ and $\mathbf{\hat{q}}$ corresponds to the \textit{true} flows, then $\mathbf{\tilde{q}} = \mathbf{\hat{q}} + \zeta$ holds.
\end{theorem}

\begin{proof}
    The proof is given in Appendix \ref{subsection_ProofTheorem3.3}.
\end{proof}

\noindent Therefore, enforcing the equality of $\mathbf{\tilde{q}}$ and $\mathbf{\hat{q}}$ by choosing a suitable loss function will guide us to a physically correct solution for flows, as we discuss in the following section.

\subsection{Overall Model and Training} \label{subsection_OverallModelandTraining}

We obtain an overall model $f(\cdot, \Theta) = f_2 \circ f_1(\cdot, \Theta)$ leveraging GCNs and the water hydraulics. Overall dynamics and training incorporate two further novel design principles:

Since observational \emph{training}\/ data are not easily available, we do not train $f$ in the typical supervised learning sense. Available data are the true demands $\mathbf{d}^* \in \R^{N_n}$ at consumer nodes and the true heads $(h_v^*)_{v \in V_r}$ at the reservoir nodes $V_r \subset V$. This information is used together with error terms which confirm that physical constraints of the hydraulics are fulfilled, as we will detail below.

The \emph{model dynamics}\/ has to ensure that information can spread through all nodes of the WDS; therefore, before updating the model's parameters $\Theta$, for $K \in \mathbb{N}$, it is applied $K$-times to each sample of a training batch. This iterative scheme is defined as follows:
As a first step, for each sample of a batch $S_b$, we need to initialize the node features, i.e., the demands $\mathbf{D}$, and the edge features, i.e., the flows  $\mathbf{Q}$, which are the required inputs to $f_1$, and the heads $\mathbf{h}$, as the required input to $f_2$. Thus, for $k = 0$, for the pressure heads $\mathbf{h}^{(0)} = (h_v^{(0)})_{v \in V}$, we define
\begin{align*}
    h_v^{(0)} &:=
    \begin{cases}
     h_v^* \quad & \text{ if } v \in V_r \\
    0 \quad & \text{ if } v \in V \setminus V_r 
    \end{cases}
\end{align*}

\noindent for all $v \in V$. We use the true demand as the input demands $\mathbf{D}^{(0)} = (\mathbf{d_1}^{(0)}, \mathbf{d_2}^{(0)})$ and initialize the input flows $\mathbf{Q}^{(0)} = (\mathbf{q_1}^{(0)}, \mathbf{q_2}^{(0)})$ based on the water hydraulics from eq. \eqref{align_HeadLoss} (analogously to eq. \eqref{align_ComputeFlowsAndDemands}):
\begin{align*}
    d_{v1}^{(0)}, d_{v2}^{(0)} &:=
    \begin{cases}
    0 \quad & \text{ if } v \in V_r \\
    d_v^* \quad & \text{ if } v \in V \setminus V_r 
    \end{cases}
    ,\\
    q_{e1}^{(0)}, q_{e2}^{(0)} &:=
    \sgn(h_v^{(0)} - h_u^{(0)}) \cdot 
    (r_e^{-1} |h_v^{(0)} - h_u^{(0)}|)^{1/x}
\end{align*}

\noindent for all $v \in V$ and all $e = e_{vu} \in E$.
Afterwards, we use the outcomes of $f_1$ and $f = f_1 \circ f_2$ of the $k$-th iteration in the $k+1$-th iteration: For $k = 0,...,K-1$, we define
{\small
\begin{align*}
    \mathbf{\hat{q}}^{(k)} &:= f_1\left( \Big((\mathbf{d_1}^{(k)},\mathbf{d_2}^{(k)}),(\mathbf{q_1}^{(k)},\mathbf{q_2}^{(k)})\Big), \Theta \right)\\
    (\mathbf{\tilde{h}}^{(k)}, \mathbf{\tilde{d}}^{(k)}, \mathbf{\tilde{q}}^{(k)}) &:= f_2(\mathbf{h}^{(0)}, \mathbf{\hat{q}}^{(k)})\\
    \mathbf{d_1}^{(k+1)} &:= \mathbf{d}^{(0)} \text{ (true demand)}\\
    \mathbf{d_2}^{(k+1)} &:= \mathbf{\hat{d}}^{(k)}\text{ (demand from GCN layers)}\\
    \mathbf{q_1}^{(k+1)} &:= \mathbf{\hat{q}}^{(k)} \text{ (flow from GCN layers)}\\
    \mathbf{q_2}^{(k+1)} &:= \mathbf{\tilde{q}}^{(k)} \text{ (flow from physics-informed algo)},
\end{align*}
}%
\noindent where we compute the demands $\mathbf{\hat{d}}^{(k)}$ by flows $\mathbf{\hat{q}}^{(k)}$ analogously to eq. \eqref{align_ComputeFlowsAndDemands}.\footnote{
    We do not re-use the demand $\mathbf{\tilde{d}}^{(k)}$ from the physics-informed algorithm as an input node feature, since including it did not improve performance.
}
More precisely, in each iteration, we update the second node feature by using the updated demands from the output of $f_1$ while keeping the first node feature fixed to allow better information propagation between iterations. Additionally, we update the edge features by using the updated flows computed by $f_1$ and $f_2$.

After making use of the model $f(\cdot, \Theta)$ $K$ times, we update its parameters by minimizing the following \emph{loss function}\footnote{
    As true demands at the reservoirs are unknown, with a slight abuse of notation, here the demand vectors exclude all $v \in V_r$.
} through back-propagation:
\begin{align*}
    \mathcal{L} = 
    \mathcal{L} (\mathbf{d}^*, \mathbf{\hat{d}}^{(K)}) 
    + \rho
    \mathcal{L} (\mathbf{d}^*, \mathbf{\tilde{d}}^{(K)}) 
    + \delta
    \mathcal{L} (\mathbf{\hat{q}}^{(K)}, \mathbf{\tilde{q}}^{(K)}),
\end{align*}

\noindent where $\rho$ and $\delta$ are hyperparameters and

\begin{itemize}
    \item $\mathcal{L} (\mathbf{d}^*, \mathbf{\hat{d}}^{(K)})$ regresses the demands $\mathbf{\hat{d}}^{(K)}$ computed by the GCN layers $f_1$ against the true demands $\mathbf{d}^*$,
    
    \item $\mathcal{L} (\mathbf{d}^*, \mathbf{\tilde{d}}^{(K)})$ regresses the demands $\mathbf{\tilde{d}}^{(K)}$ computed by the whole model $f = f_2 \circ f_1$ against the true demands $\mathbf{d}^*$,
    
    \item $\mathcal{L} (\mathbf{\hat{q}}^{(K)}, \mathbf{\tilde{q}}^{(K)})$ regresses the flows $\mathbf{\hat{q}}^{(K)}$ computed by the GCN layers $f_1$ against the flows $\mathbf{\tilde{q}}^{(K)}$ computed by the whole model $f = f_2 \circ f_1$ (cf. theorem \ref{theorem_MotivationAlgorithm}),
    
    \item and $\mathcal{L}$ denotes the L1 loss, i.e., the mean absolute error, over all nodes $V$ and all samples $S_b$ in a batch.
\end{itemize}

\noindent It is important that  $\rho$ and $\delta$ are considerably smaller than one, because the first term allows the GCN layers to estimate an initial set of flows and the remaining terms guide the model towards a valid solution. Our methodology allows the GCN component $f_1$ to learn locally, which is augmented by the global physics-informed algorithm $f_2$. Multiple iterations of this methodology allows the local GCN to learn the global task and solve it accurately.

\begin{table}[!htbp]
\centering
\setlength{\tabcolsep}{24pt}
\renewcommand{\arraystretch}{1.1}
\resizebox{.99\columnwidth}{!}{
\begin{tabular}{lc}
    \hline
    \textbf{Hyperparameters} & \textbf{Values} \\
    \hline
    No. of GCN layers $I$ & 5 \\
    No. of MLP layers & 2 \\
    Latent dimension ($O$) & 128 \\
    No. of training epochs & 3000 \\
    Learning Rate (LR) & 0.0001 \\
    LR scheduler step size & 300 \\
    LR scheduler decay rate & 0.75 \\
    No. of training scenarios & 20 \\
    No. of training samples & 1920 \\
    No. of training iterations $K$ & [10, 15]  \\
    $\rho \quad \text{and} \quad \delta$ & 0.1 \\
    No. of evaluation iterations & 20  \\
    No. of evaluation scenarios & 30 \\
    No. of evaluation samples & 20,160 \\
\end{tabular}
}
\caption{Hyperparameters used for training and evaluation}
\label{table: hyperparameters}
\end{table}

\begin{table*}[!htbp]
\centering
\setlength{\tabcolsep}{12pt}
\renewcommand{\arraystretch}{1.2}
\resizebox{.99\textwidth}{!}{
\begin{tabular}{lccccc}
    \hline
    \textbf{WDS} & \textbf{Hanoi} & \textbf{Fossolo} & \textbf{Pescara} & \textbf{L-Town Area-C} & \textbf{Zhi Jiang} \\
    \hline
    No. of junctions & 32 & 37 & 71 & 93 & 114 \\
    No. of links & 68 & 116 & 198 & 218 & 328 \\
    No. of reservoirs & 1 & 1 & 3 & 1 & 1\\
    Diameter & 13 & 8 & 20 & 20 & 24 \\
    Node degree (min, mean, max) & (2, 4.25, 8) & (2, 6.27, 8) & (2, 5.52, 10) & (2, 4.69, 8) & (2, 5.75, 8) \\
\end{tabular}
}
\caption{Attributes of WDS used for experiments}
\label{table: WDS attributes}
\end{table*}

\section{Experiments} \label{sec: experiments}

We evaluate our methodology on a number of real-world WDS. To the best of our knowledge, there is no comparable ML approach for the task of state, i.e., head and flow, estimation. Hence, we compare our results with the (computationally demanding) \enquote{ground truth} of the hydraulic simulator EPANET \cite{rossman2020epanet}. 
EPANET can run two types of simulations, demand driven (DD) and pressure dependent demand (PDD). In DD simulation, demands of all consumers are ensured, while in PDD, all demands may not be met depending on the drop in pressure. Our current methodology emulates DD simulation, although it can be modified to emulate PDD simulation. 

\subsection{Datasets}

A WDS can consist of different types of junctions (reservoirs, consumers, tanks) and links (pipes, valves, pumps). A simple WDS normally has a single reservoir, many consumers and pipes. Adding a second reservoir increases the complexity of the problem. If there are valves, the WDS needs to be broken down into parts to estimate the states (heads, flows). Tanks and pumps add a temporal dimension to the dataset. For our experiments, we use datasets based on WDS with a single or multiple reservoirs and where all links are pipes (cf. table \ref{table: WDS attributes}). 

Hanoi is one of the most popular WDS being used in the domain of WDS. Different scenarios are available with varying demand patterns, diameters, length and roughness coefficients \cite{vrachimis2018leakdb}. L-TOWN is  a well known large WDS with one available demand pattern  \cite{vrachimis2020battledim}. We use only Area-C of L-TOWN, where we treat it as a separate WDS by attaching a reservoir of head 200$m$ and changing the demand multiplier to 5. Moreover, Fossolo, Pescara and Zhi Jiang are real world international WDS datasets \cite{fossolo, pescara, zhijiang} with no demand patterns available.

We are particularly interested in emulations which can address different WDS configurations (regarding link attributes such as diameter) and varying demands within a single model.
To eliminate the need for re-training we thus train for multiple scenarios for each WDS. For Hanoi, such scenarios are already available. For L-Town Area-C, there is one demand pattern available. For the other three WDS, we generate demands by sampling from a normal distribution $\mathcal{N}(0,\,1)$. Except for Hanoi, for each scenario, we add variation to the demands by adding noise sampled from $\mathcal{N}(0,\,0.1)$. Moreover, we vary the diameters by adding noise sampled from $\mathcal{N}(0,\,1/30)$. Additionally, we use different seeds for each scenario. Note that varying the diameter gives enough variation for the model to generalize to varying values of $r_e$,  as for $e \in E$, instead of the magnitudes $l_e, d_e$,  and $c_e$, the model gets $r_e$ as input (cf. eq. \eqref{align_ConstantR}).

\begin{table*}[!t]
\centering
\setlength{\tabcolsep}{12pt}
\renewcommand{\arraystretch}{1.2}
\resizebox{.99\textwidth}{!}{
\begin{tabular}{lccccc}
    \hline
    \textbf{WDS} & \textbf{Hanoi} & \textbf{Fossolo} & \textbf{Pescara} & \textbf{L-Town Area-C} & \textbf{Zhi Jiang} \\
    \hline
    & \multicolumn{5}{c}{\textbf{Time in seconds.}} \\
    \hline
    EPANET  & 22.87 & 719.86 & 1318.86 & 1380.40 & 274.80 \\
    GCN Model & 5.98 & 7.24 & 10.44 & 11.12 & 15.58 \\
% \end{tabular}
% }
% \caption{Evaluation times (seconds) to simulate/emulate 20,160 samples}
% \label{table: time results}
% \end{table*}

% \begin{table*}[t]
% \centering
% \setlength{\tabcolsep}{24pt}
% \renewcommand{\arraystretch}{1.1}
% \resizebox{.99\textwidth}{!}{
% \begin{tabular}{l|c|c|c|c|c}
%     MRAE (\%) & Hanoi & Fossolo & Pescara & L-Town Area-C & Zhi Jiang \\
    \hline
    & \multicolumn{5}{c}{\textbf{MRAE on estimated vs true demands.} (\%)} \\
    \hline
    \textit{All samples} & & & & & \\
    EPANET & 0.001 $\pm$ 0.0005 & 0.115 $\pm$ 0.022 & 0.157 $\pm$ 0.080 & 0.362 $\pm$ 0.349 & 0.013 $\pm$ 0.002 \\
    GCN Model & 0.179 $\pm$ 0.200 & 0.324 $\pm$ 0.050 & 2.135 $\pm$ 1.614 & 1.004 $\pm$ 0.622 & 0.415 $\pm$ 0.067  \\
    \textit{Excluding 5\% outliers} & & & & & \\
    EPANET & 0.001 $\pm$ 0.0004 & 0.113 $\pm$ 0.019 & 0.149 $\pm$ 0.075 & 0.310 $\pm$ 0.272 & 0.013 $\pm$ 0.002 \\
    GCN Model & 0.147 $\pm$ 0.058 & 0.316 $\pm$ 0.034 & 1.907 $\pm$ 1.123 & 0.891 $\pm$ 0.321 & 0.404 $\pm$ 0.048  \\
    
% \end{tabular}
% }
% \caption{MRAE on estimated vs true demands across nodes and 20,160 samples.}
% \label{table: mrae demands}
% \end{table*}

% \begin{table*}[t]
% \centering
% \begin{tabular}{l|c|c|c|c|c}
%     MRAE (\%) & Hanoi & Fossolo & Pescara & L-Town Area-C & Zhi Jiang \\
    \hline
    & \multicolumn{5}{c}{\textbf{MRAE for GCN Model vs EPANET on flows and heads.} (\%)} \\
    \hline
    \textit{All samples} & & & & & \\
    Flows & 0.295 $\pm$ 1.936 & 3.972 $\pm$ 1.586 & 4.663 $\pm$ 8.400 & 0.336 $\pm$ 0.335 & 0.859 $\pm$ 0.444 \\
    Heads & 0.003 $\pm$ 0.003 & 0.002 $\pm$ 0.001 & 0.010 $\pm$ 0.004 & 0.005 $\pm$ 0.004 & 0.015 $\pm$ 0.013 \\
    \textit{Excluding 5\% outliers} & & & & & \\
    Flows & 0.144 $\pm$ 0.150 & 3.757 $\pm$ 1.290 & 3.240 $\pm$ 2.080 & 0.278 $\pm$ 0.115 & 0.794 $\pm$ 0.345 \\
    Heads & 0.002 $\pm$ 0.001 & 0.002 $\pm$ 0.001 & 0.009 $\pm$ 0.003 & 0.004 $\pm$ 0.003 & 0.013 $\pm$ 0.009 \\
    \end{tabular}
}
\caption{Results of the evaluations on 20,160 samples.}
\label{table: results}
\end{table*}

\subsection{Training Setup}

We implement all models in Pytorch and train them using the ADAM optimizer. We do not use bias in any of the layers. Hyperparameters are identical for all WDS (see ArXiv version). Since our model uses a physics-informed algorithm, the error does not increase with further iterations after convergence. Therefore, when choosing the number of GCN layers $I$ and iterations $K$, we only have to ensure that $I \cdot K$ is sufficiently larger than the diameter of the WDS. We use 20 scenarios for training with two days of data each. The sampling rate is every 30 minutes (48 samples per day) and 60:20:20 train:validation:test splits are used. We do not normalize the data in order to preserve the hydraulic relationships between demands, flows and heads. During training, we vary the number of iterations $K$ randomly within a range for every epoch to prevent over-fitting. We use learning rate scheduler and gradient clipping for smoother training. EPANET simulations are carried out using the WNTR python library \cite{klise2018overview}.

\begin{figure}[!htbp]
\centering
\includegraphics[]{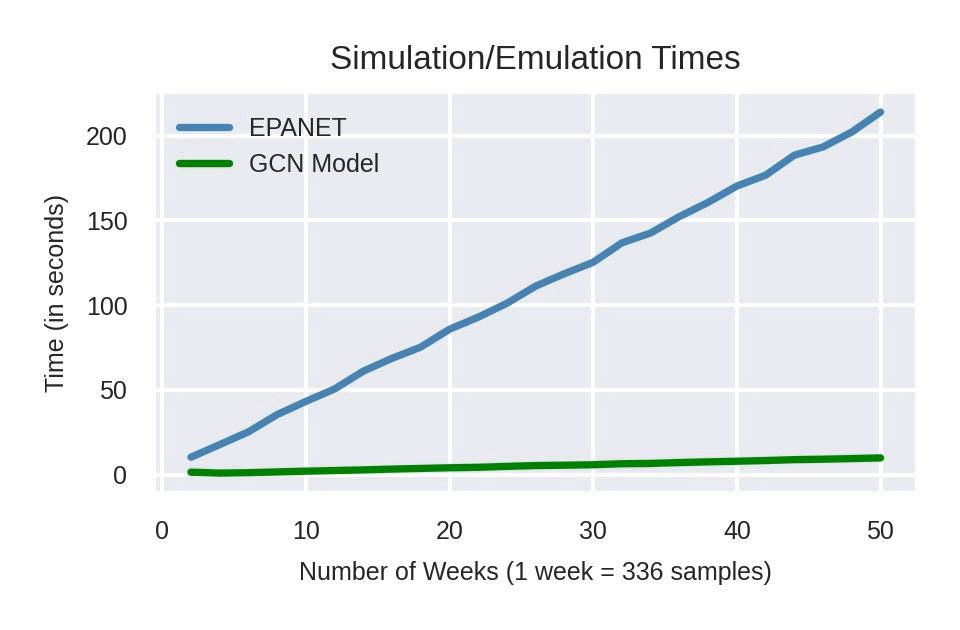} 
\caption{Comparison of simulation/emulation times on L-TOWN Area-C for EPANET and GCN Model.}
\label{fig: times}
\end{figure}

\subsection{Results and Analysis}

We evaluate all models on entirely unseen scenarios. Each of these 30 scenarios consist of 14 days of data (20,160 samples). The evaluations are done on NVIDIA GeForce RTX 4090. We compare the estimated values of demands against the true demands by computing the mean relative absolute error \footnote{We do not use the conventional mean absolute error (MAE) since our data is not normalized, i.e. heads generally have big values $\gg 0$, while flows and demands are much smaller (in decimals). Hence, MAE for flows and demands would be very low but unjust.}, given by
\begin{equation*}
    \text{MRAE}(S) =  
    \textstyle \frac{1}{N_s \cdot N_s}
    \sum_{\iota = 1}^{N_s}
    \sum_{v \in V} 
    \frac{| {d}_{\iota v}^* - {\hat{d}}_{\iota v} |}{| {d}_{\iota v}^* |}
\end{equation*}
\noindent over all samples $S = \{ (\mathbf{d}_{\iota}^*,\mathbf{\hat{d}}_{\iota}) ~|~ \iota = 1,...,N_s \}$ (cf. table \ref{table: results}). Since there are no true values available for heads and flows, we compare our results with those from the EPANET simulator using an analogous formula (cf. table \ref{table: results}). We also compute the conformity $C$ of the results to eq. \eqref{align_HeadLoss} by
\begin{equation*}
    C := 
    \textstyle \sum_{\iota = 1}^{N_s}
    \sum_{v \in V; u \in \mathcal{N}(v)} 
    \mathrm{sgn}(\tilde{h}_{\iota v} - \tilde{h}_{\iota u}) - \mathrm{sgn}(\tilde{q}_{e_{\iota vu}}),
\end{equation*}

\noindent leading to zero error for all experiments (hence not reported in tables). We achieve very good accuracy on head estimation and demand reconstruction. Most importantly, we report evaluation times, which are reduced by orders of magnitude for larger WDS as compared to EPANET (cf. table \ref{table: results}). Moreover, these do not increase drastically with the the number of samples (cf. fig. \ref{fig: times}). We added noise sampled from $\mathcal{N}(0,\,0.1)$ to the demands to create different scenarios. Great evaluation results on 30 unseen scenarios show that the model can generalize to approximately 30 percent change in demands. We also varied diameters by adding noise sampled from $\mathcal{N}(0,1/30)$, meaning that the model can also generalize to up to 10 percent change in diameters. These results clearly support our model as a viable alternative to the hydraulic simulator for WDS planning and expansion.

\begin{figure}[!htbp]
\centering
\includegraphics[]{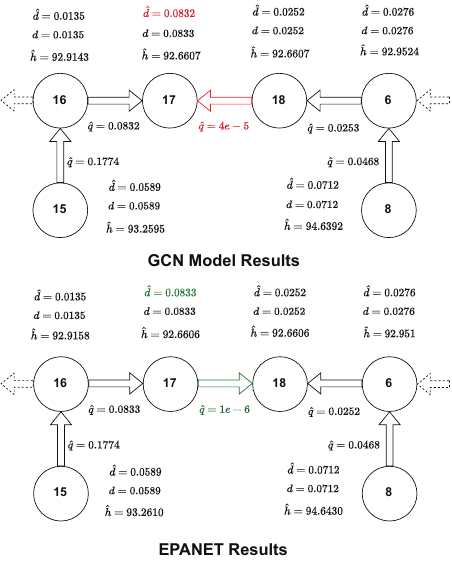} 
\caption{Limitation of the model in some cases demonstrated by an example displaying part of Hanoi WDS. The flow estimation between nodes 18 and 17 is erroneous which does induce some error in demands and heads, but the magnitude of that error is far less pronounced as in case of flows.}
\label{fig: hanoi_analysis}
\end{figure}

\subsection{Limitations}

We empirically evaluated the model by adding noise to the diameters up to standard deviation $\sigma=0.1$ and observed that the model has less than 5\% error up to $\sigma=0.08$, well beyond $\sigma=0.033$ used for training (cf. fig. \ref{fig: robustness}). We also investigated the comparatively higher standard deviation of errors in case of flows. We discovered that in some outlier samples, our model estimates some small flows erroneously. Such a sample will exhibit very high error for some flows but that will not affect the demand and head estimation by the same magnitude. This is illustrated with an example in fig. \ref{fig: hanoi_analysis}. As can be seen, the flow estimation between nodes 18 and 17 is erroneous which does induce some error in demands and heads, but the magnitude of that error is far less pronounced than the flow. Hence, we include a set of results in table \ref{table: results} excluding 5 percent outliers showing significantly lower standard deviations.

\section{Social Impact}

Sustainable urban development is not possible without efficient and timely WDS planning and expansion. Scientists at water institutes face the challenge of quick decision making on a daily basis as well as long-term planning of WDS refurbishing and extension in the light of deep uncertainties. Since hardly any new city starts from scratch, most of those decisions are about modifications to or expansion of an existing WDS. Given the structure and parameters (reservoirs, demands, pipe lengths, diameters, roughness etc.) of a WDS, even changing or adding a few new pipes can require a multitude of simulations, which is extremely costly using current  hydraulic simulation.  Therefore, a faster DL alternative is needed. We expect that our proposed model can be directly applied to such tasks  and thus can contribute to sustainable development of critical infrastructure in cities.

\begin{figure}[!tb]
\centering
\includegraphics[]{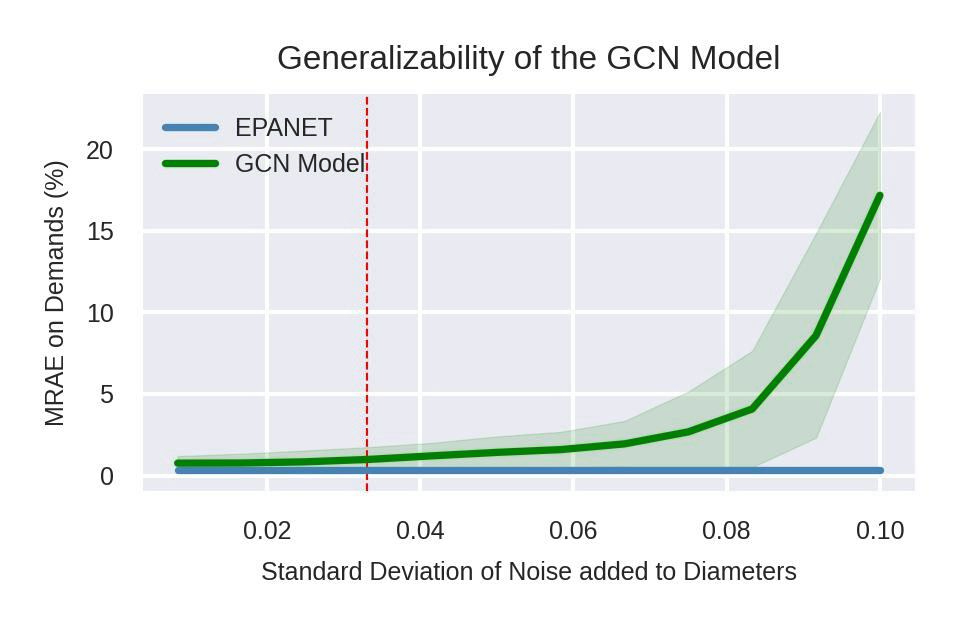} 
\caption{MRAE on Demands (L-TOWN Area-C) with increasing $\sigma$ of noise added to diameters. Red line indicates the maximum $\sigma$ used for training.}
\label{fig: robustness}
\end{figure}

\section{Conclusion and Future Work}

We present a physics-informed DL solution to the task of state estimation in WDS. The task is of utmost importance for planning and expansion of WDS. The hydraulic solver EPANET is the current go to solution for this task for researchers, scientists and engineers in the field of WDS. However, the solver suffers from long computation times since it scales non-linearly with the size of WDS. Moreover, even the slightest changes in the configuration of WDS requires a complete re-run. We utilize GCN layers and a unique physics-informed algorithm to build a DL alternative that is vastly faster than the hydraulic simulator. To the best of our knowledge, this is the first DL approach that does not use any additional information to solve the task. We use a limited number of GCN layers iteratively thus keeping the model size small. Unlike conventional ML tasks, we infer two features (heads, flows) from a given feature (demand) and achieve great accuracy using hydraulic principles.    

In the future, we plan to extend the proposed methodology to adapt to more complex WDS. For instance, a WDS with valves adds a break to the continuity of the hydraulic principles. A pump adds a Markov property to the flows and will need further adjustments to our model. The current DD simulations can be modified to run PDD simulations by adding certain constraints. On a graph level, the impact of changes to WDS structure on the performance of our methodology needs to be explored. Finally, a model able to generalize across different WDS will be the ideal solution.

\section*{Acknowledgements} 
We gratefully acknowledge funding from the European Research Council (ERC) under the ERC Synergy Grant Water-Futures (Grant agreement No. 951424) and the research training group “Dataninja” (Trustworthy AI for Seamless Problem Solving: Next Generation Intelligence Joins Robust Data Analysis) funded by the German federal state of North Rhine-Westphalia.

\bibliography{aaai24}

\clearpage

\appendix

\section{Theoretical Background}
\label{section_TheoreticalBackground}

\subsection{Proof of Lemma 3.1}
\label{subsection_ProofLemma3.1}

\begin{lemma}
    Given $N_n$ nodes $v \in V$ and $N_e > N_n$ edges $e \in E$, there is no unique solution to eq. \eqref{align_MassBalance}.
\end{lemma}

\begin{proof}
    Eq. \eqref{align_MassBalance} can be written as a system of linear equations of the kind $-\mathbf{A} \mathbf{q}^* = \mathbf{d}^*$, where the entry $a_{i_vi_e}$ of the matrix $\mathbf{A} \in \R^{N_n \times N_e}$ equals one if the $i_e$-th edge is the edge $e_{v u}$ of the $i_v$-th node $v \in V$ and some neighbor $u \in \N(v)$ and zero else. As $\rk(\mathbf{A}) \leq \min\{N_n,N_e\} = N_n < N_e$, this system is not uniquely solvable. 
\end{proof}

\subsection{Proof of Theorem 3.2}
\label{subsection_ProofTheorem3.2}

In order to prove theorem \ref{theorem_ConvergenceOfAlgorithm}, we need to introduce the concept of paths:

\begin{definition}[Paths]
\label{definition_Paths}
    We denote the tuple $(v_0, v_1, ..., v_L)$ as the \textit{path} from $v_r \in V_r \subset V$ to $v \in V$, if
    \begin{itemize}
        \item $v_0, ..., v_L \in V$,
        \item for all $l_1, l_2 \in \{0,...,L\}$ for which $l_1 \neq l_2$ holds, also $v_{l_1} \neq v_{l_2}$ holds,
        \item $v_0 = v_r$,
        \item $v_{l+1} \in \N(v_l)$ for all $l = 0, ..., L-1$, and
        \item $v_L = v$
    \end{itemize}
    
    \noindent holds. Moreover, we denote 
    \begin{itemize}
        \item the set of paths from $v_r \in V_r$ to $v \in V$ by $\Pa(v_r,v)$,
        \item the set of paths from any $v_r \in V_r$ to $v \in V$ by $\Pa(v) := \{ p ~|~ v_r \in V_r, p \in \Pa(v_r,v) \}$,
    \end{itemize}
    
    \noindent and 
    \begin{itemize}
        \item the length of a path $p = (v_0, ..., v_L)$ by $L(p) := L$,
        \item the length of the longest path from any $v_r \in V_r$ to $v \in V$ by $L_{\max}(v) := \max_{p \in \Pa(v)} L(p)$,
        \item the length of the longest path from any $v_r \in V_r$ to any $v \in V$ by $L_{\max} = \max_{v \in V} L_{\max}(v)$
    \end{itemize}
    
    \noindent and the set of path lengths of from any $v_r \in V_r$ to $v \in V$ by $\Le(v) := \{L(p) ~|~ p \in \Pa(v) \}$.
\end{definition}

\noindent Based on this, after defining how we initialize the algorithm of theorem \ref{theorem_ConvergenceOfAlgorithm}/\ref{theorem_ConvergenceOfAlgorithm_appendix}, we divide its proof in several lemmata and corollaries.
For simplicity, for $e = e_{vu} \in E$, we define 
\begin{align*}
    w_e := \mathrm{ReLU}\left(- r_e  \sgn({\hat{q}}_e) |{\hat{q}}_e|^x \right)
\end{align*}

\noindent such that the edge messages $m_e^{(j)}$ (cf. eq. \eqref{align_EdgeMessageGeneration_PhysicsInformed}) become $m_e^{(j)} = {\tilde{h}}_u^{(j)} - w_e$ for all $j \in \{0,...,J-1\}$. Especially to mention, $w_e$ does \textit{not} depend on $j$.

\begin{assumption}[Initialization]
\label{assumption_Initialization}
     Let $V_r \subset V$ be a subset of non-connecting nodes of the WDS with associated values $(h_v^*)_{v \in V_r}$, $L_{\max}$ the length of the longest path from any instance $v_r \in V_r$ to any instance $v \in V$ and define $c := \min_{v_r \in V_r} h_{v_r}^* - L_{\max} \cdot \max_{e \in E} w_e$.
    Then we initialize $\mathbf{\tilde{h}}^{(0)} = \mathbf{h} \in \R^{N_n}$ according to
    \begin{align*}
        h_v
        :=
        \begin{cases}
            h_v^* & \text{ if } v \in V_r \\
            c & \text{ if } v \in V \setminus V_r.
        \end{cases}
    \end{align*}
\end{assumption}

\begin{theorem}
\label{theorem_ConvergenceOfAlgorithm_appendix}
    If assumption \ref{assumption_Initialization} holds, the recursive algorithm defined by 
    \begin{align}
    \label{align_Algorithm_nested}
        \tilde{h}_v^{(j)} 
        := 
        \max\{ 
        \tilde{h}_v^{(j-1)}, 
        \max_{u \in \N(v)}
        \tilde{h}_u^{(j-1)} - w_{e_{vu}} 
        \}
    \end{align}
    
    \noindent converges in at most $J = L_{\max}$ steps.
\end{theorem}

\noindent First of all, we explore how information is passed from a node $v_r \in V_r$ along any path in the WDS. In practise, where $V_r$ corresponds to the reservoir nodes of a WDS, the information equals head values which are passed from a reservoir node $v_r \in V_r$ to any other node $v \in V$ in the WDS.

\begin{lemma}
\label{lemma_ConvergenceOfSinglePaths}
    Let assumption \ref{assumption_Initialization} hold.
    Let $p = (v_0, ..., v_L)$ be a path from $v_r \in V_r$ to $v \in V$ which is not intersected by any other path, i.e., for which $\N(v_0) = \{v_1\}$, $\N(v_l) = \{v_{l-1}, v_{l+1}\}$ for all $l = 1,...,L-1$ and $\N(v_L) = \{v_{L-1}\}$ holds.
    Moreover, assume that (a) $v_1, ..., v_L \in V \setminus V_r$ holds.
    %and (b)
    %
    %\begin{align*}
    %    h_{v_r}^* - \sum_{l = 1}^{L} w_{e_{v_l v_{l-1}}} \geq 0
    %\end{align*}
    %
    %\noindent hold.
    Then, we can conclude that
    \begin{align}
    \label{align_AlgorithmForSinglePaths_final}
        \tilde{h}_{v_l}^{(j)}
        =
        \begin{cases}
            \tilde{h}_{v_{l}}^{(0)} 
            & \text{ if } j < l \text{ or } l = 0\\
            \tilde{h}_{v_{l-1}}^{(l-1)} - w_{e_{v_l v_{l-1}}} 
            & \text{ if } j \geq l \neq 0
        \end{cases}
    \end{align}
    
    \noindent holds for all $l = 0,...,L$ and all $j \in \mathbb{N}$.
\end{lemma}

\begin{proof}
    We prove lemma \ref{lemma_ConvergenceOfSinglePaths} by induction to $j \in \mathbb{N}$.
    Beforehand, note that as the path is not intersected, eq. \eqref{align_Algorithm_nested} simplifies to
    \begin{align}
    \label{align_AlgorithmForSinglePaths}
    \begin{split}
        \tilde{h}_0^{(j)}
        =&
        \max \{
        \tilde{h}_0^{(j-1)},
        \tilde{h}_{1}^{(j-1)} - w_{0 1}
        \},
        \\
        \tilde{h}_l^{(j)}
        =&
        \max \{
        \tilde{h}_l^{(j-1)},
        \tilde{h}_{l-1}^{(j-1)} - w_{l (l-1)},
        \tilde{h}_{l+1}^{(j-1)} - w_{l (l+1)}
        \},
        \\
        &\text{ for all }l = 1,...,L-1 \text{ and }
        \\
        \tilde{h}_L^{(j)}
        =&
        \max \{
        \tilde{h}_L^{(j-1)},
        \tilde{h}_{L-1}^{(j-1)} - w_{L (L-1)}
        \},
    \end{split}
    \end{align}
    
    \noindent where for simplicity, in this proof, we use the notation $\tilde{h}_l^{(j)} := \tilde{h}_{v_l}^{(j)}$ for all $l = 0,...,L$, and $w_{e_{v_l v_{l-1}}} := w_{l (l-1)}$ for all $l = 1,...,L$, etc.
    
    Moreover, before starting with the induction base, we also discuss several estimations on the constant $c$ from assumption \ref{assumption_Initialization}: 
    By definition, $w_e \geq 0$ and thus, $-w_e \leq 0$ holds for all $e \in E$.
    Therefore, for any $e \in E$,
    \begin{align}
    \label{align_cEstimation_1}
        c - w_e
        \leq 
        c
    \end{align}
    
    \noindent holds. Even more, for any non-empty subset $\tilde{E} \subset E$ for which $|\tilde{E}| \leq L_{\max}$ holds and for any $v_r \in V_r$, we obtain
    \begin{align}
    \label{align_cEstimation_2}
        c
        =
        \underbrace{
        \min_{u_r \in V_r} h_{u_r}^* 
        }_{\leq h_{v_r}^* }
        -
        \underbrace{
        \underbrace{L_{\max}}_{\geq |\tilde{E}| \geq 1} 
        \cdot 
        \underbrace{
        \max_{\tilde{e} \in E} w_{\tilde{e}}
        }_{\geq \max_{\tilde{e} \in \tilde{E}} w_{\tilde{e}} \geq 0}
        }_{\geq \sum_{\tilde{e} \in \tilde{E}} w_{\tilde{e}}}
        \leq
        h_{v_r}^*
        -
        \sum_{\tilde{e} \in \tilde{E}} w_{\tilde{e}}
    \end{align}
    
    \noindent and even more, for an $e \in \tilde{E}$,
    \begin{align}
    \label{align_cEstimation_3}
        c
        \leq 
        h_{v_r}^*
        -
        \underbrace{
        \sum_{\tilde{e} \in \tilde{E}} \underbrace{w_{\tilde{e}}}_{\geq 0}
        }_{\geq w_e}
        \leq
        h_{v_r}^*
        -
        \underbrace{w_e}_{\geq 0}
        \leq 
         h_{v_r}^*
        .
    \end{align}
    
    \begin{remark}[Further estimations on $c$]
    \label{remark_cEstimation}
        Even more, by definition of $L_{\max} \geq 1$, similar arguments show that for \textit{any} $e \in E$ and for any $v_r \in V_r$, we obtain
        \begin{align*}
            c - w_e
            \leq
            c
            =&~
            \min_{u_r \in V_r} h_{u_r}^* 
            -
            L_{\max}
            \cdot
            \max_{\tilde{e} \in E} w_{\tilde{e}}
            \\
            \leq&~
            h_{v_r}^* 
            -
            L_{\max}
            \cdot 
            w_e
            \\
            \leq&~
            h_{v_r}^* 
            -
            w_e
            \\
            \leq&~
            h_{v_r}^*
            .
        \end{align*}
        % --- LONG VERSION (if they wanna know what "similar arguments" are
        %
        %\begin{align*}
        %    c
        %    =
        %    \underbrace{
        %    \min_{u_r \in V_r} h_{u_r}^* 
        %    }_{\leq h_{v_r}^* }
        %    \underbrace{
        %    -
        %    \underbrace{L_{\max}}_{\geq 1} 
        %    \cdot 
        %    \underbrace{
        %    \max_{\tilde{e} \in E} w_{\tilde{e}}
        %    }_{\geq w_e \geq 0}
        %    }_{\leq - L_{\max} \cdot w_e}
        %    \leq
        %    h_{v_r}^* 
        %    -
        %    L_{\max}
        %    \cdot 
        %    w_e
        %\end{align*}
        %
        %\noindent and even more,
        %
        %\begin{align*}
        %    c
        %    \leq 
        %    h_{v_r}^* 
        %    \underbrace{
        %    -
        %    \underbrace{L_{\max}}_{\geq 1} 
        %    \cdot 
        %    \underbrace{w_e}_{\geq 0}
        %    }_{\leq - w_e}
        %    \leq 
        %    h_{v_r}^*
        %    - 
        %    w_e
        %    \leq 
        %    h_{v_r}^*.
        %\end{align*}
    \end{remark}
    
    \noindent In the rest of the proof, we will mostly make use of these estimations and two other arguments (cf. remark \ref{remark_Arguments}).
    \\
    \\\textbf{Induction base:} \textit{Case 1:}
    If $j = 1$ and $l = 0$, 
    the node $v_l = v_0 = v_r$ is a reservoir node while $v_{l+1} = v_1$ is not by assumption (a).
    Thus,
    by 
    (\romannumeral 1) eq. \eqref{align_AlgorithmForSinglePaths},
    (\romannumeral 2) assumption \ref{assumption_Initialization} and
    (\romannumeral 3) eq. \eqref{align_cEstimation_1} to \eqref{align_cEstimation_3},
    \begin{align*}
        \tilde{h}_{l}^{(j)}
        =&~
        \tilde{h}_0^{(1)}
        \\
        \overset{\text{(\romannumeral 1)}}{=}&~
        \max \{
        \tilde{h}_{0}^{(0)},
        \tilde{h}_{1}^{(0)} - w_{0 1}
        \}
        \\
        \overset{\text{(\romannumeral 2)}}{=}&~
        \max \{
        h_{v_r}^{*},
        c - w_{0 1}
        \}
        \\
        \overset{\text{(\romannumeral 3)}}{=}&~
        h_{v_r}^{*}
        = 
        \tilde{h}_{0}^{(0)} = \tilde{h}_{l}^{(0)}
    \end{align*}
    
    \noindent holds, which proves the induction base for $l = 0$.
    
    \begin{remark}
    \label{remark_Arguments}
        From now on, we will use the arguments 
        \\\noindent (\romannumeral 1), i.e., the usage of eq. \eqref{align_AlgorithmForSinglePaths},
        \\\noindent(\romannumeral 2), i.e., the usage of assumption \ref{assumption_Initialization},
        and
        \\\noindent(\romannumeral 3), i.e., the usage of at least one of the equations \eqref{align_cEstimation_1}, \eqref{align_cEstimation_2} or \eqref{align_cEstimation_3},
        multiple times.
        
        Moreover, for simplicity, we will not distinguish between the cases $1 \leq l \leq L-1$ and $l = L$ when making use of eq. \eqref{align_AlgorithmForSinglePaths} and use the second eq. for $\tilde{h}_l^{(j)}$ also for $l = L$ with a slight abuse of notation, i.e., while ignoring that $v_L$ does only have one neighbor.
        That will not change the results, just allows us to consider less edge case. One can think of replacing the path $(v_0,...,v_L)$ by a path $(v_0, ..., v_L, v_{L+1})$, where $v_L = v$ still holds and $v_{L+1}$ corresponds to any other node with no other neighbor than $v_L$ and in which we are not interested in.
    \end{remark}
    
    \noindent \textit{Case 2:}
    If $j = 1$ and $j < l$, 
    we obtain $l, l \pm 1 \geq 1$. Therefore, 
    the nodes $v_{l-1}, v_l$ and $v_{l+1}$ are no reservoir nodes by assumption (a). 
    Thus,
    \begin{align*}
        \tilde{h}_l^{(j)}
        =&~
        \tilde{h}_l^{(1)}
        \\
        \overset{\text{(\romannumeral 1)}}{=}&~
        \max \{
        \tilde{h}_l^{(0)},
        \tilde{h}_{l-1}^{(0)} - w_{l (l-1)},
        \tilde{h}_{l+1}^{(0)} - w_{l (l+1)}
        \}
        \\
        \overset{\text{(\romannumeral 2)}}{=}&~
        \max \{
        c,
        c - w_{l (l-1)},
        c - w_{l (l+1)}
        \}
        \\
        \overset{\text{(\romannumeral 3)}}{=}&~
        c
        =
        \tilde{h}_l^{(0)}
    \end{align*}
    
    \noindent holds, which proves the induction base for $l > j = 1$.
    \\
    \\\noindent \textit{Case 3:}
    If $j = 1$ and $j \geq l \neq 0$, 
    we obtain $l = 1$. Therefore, 
    the node $v_{l-1} = v_0 = v_r$ is a reservoir node while $v_l = v_1$ and $v_{(l+1)} = v_2$ is not by assumption (a). 
    Thus,
    \begin{align*}
        \tilde{h}_l^{(j)}
        =&~
        \tilde{h}_1^{(1)}
        \\
        \overset{\text{(\romannumeral 1)}}{=}&~
        \max \{
        \tilde{h}_1^{(0)},
        \tilde{h}_{0}^{(0)} - w_{1 0},
        \tilde{h}_{2}^{(0)} - w_{1 2}
        \}
        \\
        \overset{\text{(\romannumeral 2)}}{=}&~
        \max \{
        c,
        \tilde{h}_{v_r}^{*} - w_{1 0},
        c - w_{1 2}
        \}
        \\
        \overset{\text{(\romannumeral 3)}}{=}&~
        \tilde{h}_{v_r}^{*} - w_{1 0}
        =~
        \tilde{h}_{0}^{(0)} - w_{1 0}
        =~
        \tilde{h}_{l-1}^{(j-1)} - w_{l (l-1)}
    \end{align*}
    
    \noindent holds, which proves the induction base for $l \leq j = 1$.
    \\
    \\\textbf{Induction hypothesis:}
    We can assume that 
    \begin{align*}
        \tilde{h}_{\hat{l}}^{(\hat{\jmath})}
        =
        \begin{cases}
            \tilde{h}_{\hat{l}}^{(0)} 
            & \text{ if } \hat{\jmath} < \hat{l} \text{ or } \hat{l} = 0\\
            \tilde{h}_{\hat{l}-1}^{(\hat{l}-1)} - w_{\hat{l} (\hat{l}-1)} 
            & \text{ if } \hat{\jmath} \geq \hat{l} \neq 0 %\\
            %\tilde{h}_{\hat{l}}^{(\hat{l})} & \text{ if } \hat{\jmath} > \hat{l} \neq 0 \\
        \end{cases}
    \end{align*}
    
    \noindent holds for all $\hat{l} = 0,...,L$ and all $\hat{\jmath} = 1, ..., j-1$.
    \\
    \\\textbf{Induction step:} 
    We need to show that eq. \eqref{align_AlgorithmForSinglePaths_final} holds for all $l = 0,...,L$ and for a $j \in \mathbb{N}_{> 1}$, given that the induction hypothesis holds for all $\hat{l} = 0,...,L$ and for all $\hat{\jmath} = 1, ..., j-1$. Based on this hypothesis, we first of all prove the following in-lemma lemma:
    
    \begin{lemma}
    \label{lemma_ReservoirIsSource}
        In the setting of the induction step of lemma \ref{lemma_ConvergenceOfSinglePaths}, for all $l = 1,...,\min \{j,L\}$, we obtain
        \begin{align}
        \label{align_ReservoirIsSource}
            \tilde{h}_{l-1}^{(l-1)} - w_{l (l-1)}
            =
            h_{v_r}^* - \textstyle \sum_{\hat{l} = 1}^{l} w_{\hat{l} (\hat{l}-1)}.
        \end{align}
    \end{lemma}
    
    \begin{proof}
        We prove lemma \ref{lemma_ReservoirIsSource} by induction to $l \in \{1,...,\min \{j,L\}\}$.
        \\
        \\\textbf{Induction base:} 
        If $l = 1$, 
        \begin{align*}
            \tilde{h}_{l-1}^{(l-1)} - w_{l (l-1)}
            =&~
            \tilde{h}_{0}^{(0)} - w_{1 0}
            \\
            =&~
            \tilde{h}_{0}^{(0)} - \textstyle \sum_{\hat{l} = 1}^{1} w_{\hat{l} (\hat{l}-1)}
            \\
            \overset{\text{(\romannumeral 1)}}{=}&~
            h_{v_r}^* - \textstyle \sum_{\hat{l} = 1}^{l} w_{\hat{l} (\hat{l}-1)}
        \end{align*}
        
        \noindent surely holds. which proves the induction base.
        \\
        \\\textbf{Induction hypothesis:} 
        We can assume that
        \begin{align*}
            \tilde{h}_{\hat{l}-1}^{(\hat{l}-1)} - w_{\hat{l} (\hat{l}-1)}
            =
            h_{v_r}^* - \textstyle  \sum_{\hat{\hat{l}} = 1}^{\hat{l}} w_{\hat{\hat{l}} (\hat{\hat{l}}-1)}
        \end{align*}
        
        \noindent holds for all $\hat{l} = 1,...,l-1 ~(\leq j - 1)$.
        \\
        \\\textbf{Induction step:}
        We need to show that eq. \eqref{align_ReservoirIsSource} holds for a $l \in \{2,...,\min \{j,L\}\}$, given that the induction hypothesis holds for all $\hat{l} = 1,...,l-1$.
        
        As $2 \leq l \leq \min \{j,L\}$ holds by choice of $l$, we can choose $\hat{\jmath} = l - 1 \leq j - 1$ and $\hat{l} = l-1 \leq L -1 < L$ in the induction hypothesis of lemma \ref{lemma_ConvergenceOfSinglePaths}. Then, as $\hat{\jmath} = l - 1 \geq l - 1 = \hat{l} \geq 2 - 1 \neq 0$, (\romannumeral 4) by this hypothesis, we obtain
        \begin{align*}
            \tilde{h}_{l-1}^{(l-1)}
            =&~
            \tilde{h}_{\hat{l}}^{(\hat{\jmath})}
            \\
            \overset{\text{(\romannumeral 4)}}{=}&~
            \tilde{h}_{\hat{l}-1}^{(\hat{l}-1)} - w_{\hat{l} (\hat{l}-1)}
            \\
            =&~
            \tilde{h}_{l-2}^{(l-2)} - w_{(l-1) (l-2)}.
        \end{align*}
        
        \noindent Moreover, we can choose $\hat{l} = l-1 \leq l-1$ in the induction hypothesis of this lemma \ref{lemma_ReservoirIsSource}. Then, by (\romannumeral 5) this hypothesis, we obtain
        \begin{align*}
            \tilde{h}_{l-2}^{(l-2)} - w_{(l-1) (l-2)}
            =&~
            \tilde{h}_{\hat{l}-1}^{(\hat{l}-1)} - w_{\hat{l} (\hat{l}-1)}
            \\
            \overset{\text{(\romannumeral 5)}}{=}&~
            h_{v_r}^* - \textstyle \sum_{\hat{\hat{l}} = 1}^{\hat{l}} w_{\hat{\hat{l}} (\hat{\hat{l}}-1)}
            \\
            =&~
            h_{v_r}^* - \textstyle \sum_{\hat{l} = 1}^{l - 1} w_{\hat{l} (\hat{l}-1)}.
        \end{align*}
        
        \noindent Bringing both results together, 
        \begin{align*}
            \tilde{h}_{l-1}^{(l-1)} - w_{l (l-1)}
            =&~
            \tilde{h}_{l-2}^{(l-2)} - w_{(l-1) (l-2)} - w_{l (l-1)}
            \\
            =&~
            h_{v_r}^* 
            - \textstyle \sum_{\hat{l} = 1}^{l - 1} w_{\hat{l} (\hat{l}-1)} 
            - w_{l (l-1)}
            \\
            =&~
            h_{v_r}^* - \textstyle \sum_{\hat{l} = 1}^{l} w_{\hat{l} (\hat{l}-1)},
        \end{align*}
        
        \noindent holds, which proves the induction step.
    \end{proof}
    
    \noindent Finally, we can proceed with the induction step of lemma \ref{lemma_ConvergenceOfSinglePaths}. 
    For any $l \in \{0,...,L\}$ and any $j \in \mathbb{N}_{> 1}$, 
    by 
    (\romannumeral 4) the induction hypothesis of this lemma \ref{lemma_ConvergenceOfSinglePaths} with different choices of $\hat{l}$ and $\hat{\jmath}$ 
    and
    (\romannumeral 6) lemma \ref{lemma_ReservoirIsSource}\footnote{
        We leave it as an exercise to the reader to check if the condition $l \leq \min\{j,L\}$ required in lemma \ref{lemma_ReservoirIsSource} is satisfied by choice of $l \in\{0,...,L\}$ and its relation to $j$ in each three cases where the lemma is applied.
    },
    we obtain the following three observations for the three components of eq. \eqref{align_AlgorithmForSinglePaths}\footnote{
        Note that -- as already discussed in remark \ref{remark_Arguments} -- we technically had to distinguish between the cases $l = 0, l \in \{1,...,L-1\}$ and $l = L$ when considering eq. \eqref{align_AlgorithmForSinglePaths}, but for simplicity, we do not and remind the reader that for the edge cases, $l \pm 1$ might not be well-defined, but is also not needed. 
    }:
    \begin{align}
    \label{align_Component1}
    \begin{split}
        \tilde{h}_l^{(j-1)}
        \overset{\text{(\romannumeral 4)}}{=}&~
        \begin{cases}
            \tilde{h}_{l}^{(0)} 
            & \text{ if } j-1 < l \text{ or } l = 0\\
            \tilde{h}_{l-1}^{(l-1)} - w_{l (l-1)} 
            & \text{ if } j-1 \geq l \neq 0
        \end{cases}
        \\
        \overset{\text{(\romannumeral 6)}}{=}&~
        \begin{cases}
            \tilde{h}_{l}^{(0)} 
            & \text{ if } j-1 < l \text{ or } l = 0\\
            h_{v_r}^* - \sum_{\hat{l} = 1}^{l} w_{\hat{l} (\hat{l}-1)}
            & \text{ if } j-1 \geq l \neq 0
        \end{cases}
    \end{split}
    \end{align}
    
    \noindent (for $\hat{l} = l$ and $\hat{\jmath} = j-1 \leq j-1$),
    \begin{align}
    \label{align_Component2}
    \begin{split}
        \tilde{h}_{l-1}^{(j-1)}
        \overset{\text{(\romannumeral 4)}}{=}&~
        \begin{cases}
            \tilde{h}_{l-1}^{(0)} 
            & \text{ if } j-1 < l-1 \\
            & {\color{gray} \text{ or } l-1 = 0}\\
            \tilde{h}_{l-2}^{(l-2)} - w_{(l-1) (l-2)} 
            & \text{ if } j-1 \geq l-1 \neq 0
        \end{cases}
        \\
        \overset{\text{(\romannumeral 6)}}{=}&~
        \begin{cases}
            \tilde{h}_{l-1}^{(0)} 
            & \text{ if } j-1 < l-1 \\
            h_{v_r}^* - \sum_{\hat{l} = 1}^{l-1} w_{\hat{l} (\hat{l}-1)}
            & \text{ if } j-1 \geq l-1 \neq 0
        \end{cases}
    \end{split}
    \end{align}
    
    \noindent (for $\hat{l} = l-1$ and $\hat{\jmath} = j-1 \leq j-1$) and
    \begin{align}
    \label{align_Component3}
    \begin{split}
        \tilde{h}_{l+1}^{(j-1)}
        \overset{\text{(\romannumeral 4)}}{=}&~
        \begin{cases}
            \tilde{h}_{l+1}^{(0)} 
            & \text{ if } j-1 < l+1 \\
            & {\color{gray} \text{ or } l+1 = 0}\\
            \tilde{h}_{l}^{(l)} - w_{(l+1) l} 
            & \text{ if } j-1 \geq l+1 \neq 0
        \end{cases}
        \\
        \overset{\text{(\romannumeral 6)}}{=}&~
        \begin{cases}
            \tilde{h}_{l+1}^{(0)} 
            & \text{ if } j-1 < l+1 \\
            h_{v_r}^* - \sum_{\hat{l} = 1}^{l+1} w_{\hat{l} (\hat{l}-1)} 
            & \text{ if } j-1 \geq l+1 \neq 0
        \end{cases}
    \end{split}
    \end{align}
    
    \noindent (for $\hat{l} = l+1$ and $\hat{\jmath} = j-1 \leq j-1$). 
    
    Therefore, when considering $\tilde{h}_l^{(j)}$ for any $l \in \{0,...,L\}$ and any $j \in \mathbb{N}_{> 1}$, the cases $l = 0$, $j < l$ and $j \geq l \neq 0$ in eq. \eqref{align_AlgorithmForSinglePaths_final} need to be split in five cases:\footnote{
        Note that the case $l + 1 = 0$ will never occur as $l + 1 \geq 0 + 1 = 1$ holds. Moreover, the case $l - 1 = 0$ leads to the case $l = 1$. Then, the case $j < l = 1$, i.e., $j \leq 0$, never occurs. Thus, $j \geq l = 1$ must hold, which corresponds to all cases (considered by the cases 1 to 5).
    }
    
    \begin{enumerate}
        \item Case 1: $l = 0$,
        \item Case 2: $j - 1 < l - 1$ (i.e., $j < l$),
        \item Case 3: $l - 1 \leq j - 1 < l$ (i.e., $j = l $),
        \item Case 4: $l \leq j - 1 < l + 1$ (i.e., $j = l + 1$),
        \item Case 5: $l + 1 \leq j - 1$ (i.e., $j > l + 1$).
    \end{enumerate}
    
    \noindent In all of these cases, next to the arguments of remark \ref{remark_Arguments}, we will use
    \\\noindent (\romannumeral 4) the induction hypothesis of this lemma \ref{lemma_ConvergenceOfSinglePaths} again, where we define the variables $\hat{l} \in \{0,...L\}$ and $\hat{\jmath} \in \{1,...,j-1\}$ beforehand, 
    or, 
    \\\noindent (\romannumeral 6) the results in eq. \eqref{align_Component1} to \eqref{align_Component3} directly
    and multiple times.
    \\
    \\\noindent \textit{Case 1:} If $l = 0$, 
    the node $v_l = v_0 = v_r$ is a reservoir node while $v_{l+1} = v_1$ is not by assumption (a). 
    Thus, using $\hat{\jmath} := j - 1 \geq 2 - 1 = 1$ in (\romannumeral 4),  
    \begin{align*}
        \tilde{h}_{l}^{(j)}
        =&~
        \tilde{h}_0^{(j)}
        \\
        \overset{\text{(\romannumeral 1)}}{=}&~
        \max \{
        \tilde{h}_{0}^{(j-1)},
        \tilde{h}_{1}^{(j-1)} - w_{0 1}
        \}
        \\
        \overset{\text{(\romannumeral 4)}}{=}&~
        \max \{
        \tilde{h}_{0}^{(0)},
        \tilde{h}_{0}^{(0)} - w_{0 1}
        \}
        \\
        \overset{\text{(\romannumeral 2)}}{=}&~
        \max \{
        h_{v_r}^{*},
        c - w_{0 1}
        \}
        \\
        \overset{\text{(\romannumeral 3)}}{=}&~
        h_{v_r}^{*}
        = 
        \tilde{h}_{0}^{(0)} = \tilde{h}_{l}^{(0)}
    \end{align*}
    
    \noindent holds, which proves the induction step for $l = 0$.
    \\
    \\\textit{Case 2:} If $2 \leq j < l$, 
    we obtain $l, l \pm 1 \geq 2$. Therefore, 
    the nodes $v_{l-1}, v_l$ and $v_{l+1}$ are no reservoir nodes by assumption (a). 
    Thus, using
    eq. \eqref{align_Component1} with $j-1 < j < l$,
    eq. \eqref{align_Component2} with $j-1 < l-1$ and
    eq. \eqref{align_Component3} with $j-1 < l-1 < l+1$ 
    in (\romannumeral 6),
    \begin{align*}
        \tilde{h}_l^{(j)}
        \overset{\text{(\romannumeral 1)}}{=}&~
        \max \{
        \tilde{h}_l^{(j-1)},
        \tilde{h}_{l-1}^{(j-1)} - w_{l (l-1)},
        \tilde{h}_{l+1}^{(j-1)} - w_{l (l+1)}
        \}
        \\
        \overset{\text{(\romannumeral 6)}}{=}&~
        \max \{
        \tilde{h}_l^{(0)},
        \tilde{h}_{l-1}^{(0)} - w_{l (l-1)},
        \tilde{h}_{l+1}^{(0)} - w_{l (l+1)}
        \}
        \\
        \overset{\text{(\romannumeral 2)}}{=}&~
        \max \{
        c,
        c - w_{l (l-1)},
        c - w_{l (l+1)}
        \}
        \\
        \overset{\text{(\romannumeral 3)}}{=}&~
        c
        =
        \tilde{h}_l^{(0)}
    \end{align*}
    
    \noindent holds, which proves the induction step for $l > j$.
    \\
    \\\textit{Case 3:} If $2 \leq j = l$, 
    we obtain $l, l + 1 \geq 1$. Therefore,
    the nodes $v_l$ and $v_{l+1}$ are no reservoir nodes by assumption (a). 
    Thus, using 
    eq. \eqref{align_Component1} with $j-1 < j = l$,
    eq. \eqref{align_Component2} with $j-1 = l-1$ and
    eq. \eqref{align_Component3} with $j-1 < j = l < l+1$
    in (\romannumeral 6),
    \begin{align*}
        \tilde{h}_l^{(j)}
        \overset{\text{(\romannumeral 1)}}{=}~
        \max \{
        &\tilde{h}_l^{(j-1)},
        \tilde{h}_{l-1}^{(j-1)} - w_{l (l-1)},
        \tilde{h}_{l+1}^{(j-1)} - w_{l (l+1)}
        \}
        \\
        \overset{\text{(\romannumeral 4)}}{=}~
        \max \{
        &\tilde{h}_l^{(0)},\\
        &h_{v_r}^* - \textstyle \sum_{\hat{l} = 1}^{l-1} w_{\hat{l} (\hat{l}-1)} - w_{l (l-1)},\\
        &\tilde{h}_{l+1}^{(0)} - w_{l (l+1)}
        \}
        \\
        \overset{\text{(\romannumeral 2)}}{=}~
        \max \{
        &c,
        h_{v_r}^* - \textstyle \sum_{\hat{l} = 1}^{l} w_{\hat{l} (\hat{l}-1)},
        c - w_{l (l+1)}
        \}
        \\
        \overset{\text{(\romannumeral 3)}}{=}~
        {\color{white} \max \{}
        &h_{v_r}^* - \textstyle \sum_{\hat{l} = 1}^{l} w_{\hat{l} (\hat{l}-1)}
        =
        \tilde{h}_{l-1}^{(j-1)} - w_{l (l-1)}
    \end{align*}
    
    \noindent holds, which proves the first of the three parts of the induction step for $l \leq j$.
    \\
    \\\textit{Case 4:} If $2 \leq j = l + 1$, 
    we obtain $l + 1 \geq 1$. Therefore,
    the node $v_{l+1}$ is no reservoir nodes by assumption (a). 
    Thus, using 
    eq. \eqref{align_Component1} with $j-1 = l$,
    eq. \eqref{align_Component2} with $j-1 = l \geq l-1$ and
    eq. \eqref{align_Component3} with $j-1 = l < l+1$
    in (\romannumeral 6),
    \begin{align*}
        \tilde{h}_l^{(j)}
        \overset{\text{(\romannumeral 1)}}{=}~
        \max \{
        &\tilde{h}_l^{(j-1)},
        \tilde{h}_{l-1}^{(j-1)} - w_{l (l-1)},
        \tilde{h}_{l+1}^{(j-1)} - w_{l (l+1)}
        \}
        \\
        \overset{\text{(\romannumeral 4)}}{=}~
        \max \{
        &h_{v_r}^* - \textstyle \sum_{\hat{l} = 1}^{l} w_{\hat{l} (\hat{l}-1)},\\
        &h_{v_r}^* - \textstyle \sum_{\hat{l} = 1}^{l-1} w_{\hat{l} (\hat{l}-1)} - w_{l (l-1)},\\
        &\tilde{h}_{l+1}^{(0)} - w_{l (l+1)}
        \}
        \\
        \overset{\text{(\romannumeral 2)}}{=}~
        \max \{
        &h_{v_r}^* - \textstyle \sum_{\hat{l} = 1}^{l} w_{\hat{l} (\hat{l}-1)},\\
        &h_{v_r}^* - \textstyle \sum_{\hat{l} = 1}^{l} w_{\hat{l} (\hat{l}-1)},\\
        &c - w_{l (l+1)}
        \}
        \\
        \overset{\text{(\romannumeral 3)}}{=}~
        {\color{white} \max \{}
        &h_{v_r}^* - \textstyle \sum_{\hat{l} = 1}^{l} w_{\hat{l} (\hat{l}-1)}
        =
        \tilde{h}_{l-1}^{(j-1)} - w_{l (l-1)}
    \end{align*}
    
    \noindent holds, which proves the second of the three parts of the induction step for $l \leq j$.
    \\
    \\\textit{Case 5:} If $j > l + 1$, 
    using 
    eq. \eqref{align_Component1} with $j-1 \geq l$,
    eq. \eqref{align_Component2} with $j-1 > l \geq l-1$ and
    eq. \eqref{align_Component3} with $j-1 \geq l+1$
    in (\romannumeral 6),
    \begin{align*}
        \tilde{h}_l^{(j)}
        \overset{\text{(\romannumeral 1)}}{=}~
        \max \{
        &\tilde{h}_l^{(j-1)},
        \tilde{h}_{l-1}^{(j-1)} - w_{l (l-1)},
        \tilde{h}_{l+1}^{(j-1)} - w_{l (l+1)}
        \}
        \\
        \overset{\text{(\romannumeral 4)}}{=}~
        \max \{
        &h_{v_r}^* - \textstyle \sum_{\hat{l} = 1}^{l} w_{\hat{l} (\hat{l}-1)},\\
        &h_{v_r}^* - \textstyle \sum_{\hat{l} = 1}^{l-1} w_{\hat{l} (\hat{l}-1)} - w_{l (l-1)},\\
        &h_{v_r}^* - \textstyle \sum_{\hat{l} = 1}^{l+1} w_{\hat{l} (\hat{l}-1)}  - w_{l (l+1)}
        \}
        \\
        =~
        \max \{
        &h_{v_r}^* - \textstyle \sum_{\hat{l} = 1}^{l} w_{\hat{l} (\hat{l}-1)},\\
        &h_{v_r}^* - \textstyle \sum_{\hat{l} = 1}^{l} w_{\hat{l} (\hat{l}-1)},\\
        &h_{v_r}^* - \textstyle \sum_{\hat{l} = 1}^{l} w_{\hat{l} (\hat{l}-1)} - w_{(l+1) l} - w_{l (l+1)}
        \}
        \\
        \overset{\text{(\romannumeral 3)}}{=}~
        {\color{white} \max \{}
        &h_{v_r}^* - \textstyle \sum_{\hat{l} = 1}^{l} w_{\hat{l} (\hat{l}-1)}
        =
        \tilde{h}_{l-1}^{(j-1)} - w_{l (l-1)}
    \end{align*}
    
    \noindent holds, which proves the third of the three parts of the induction step for $l \leq j$.
    \\
    \\By this, the induction step is proved for $l \leq j$, and by that, the overall proof is completed.
\end{proof}

\begin{corollary}
\label{corollay_ConvergenceOfSinglePaths}
    In the setting of lemma \ref{lemma_ConvergenceOfSinglePaths}, we can conclude that
    \begin{align}
        \tilde{h}_{v_l}^{(j)}
        =
        \begin{cases}
            \tilde{h}_{v_{l}}^{(0)} 
            & \text{ if } j < l \text{ or } l = 0\\
            h_{v_r}^* - \textstyle \sum_{\hat{l} = 1}^{l} w_{\hat{l} (\hat{l}-1)}
            & \text{ if } j \geq l \neq 0 
        \end{cases}
    \end{align}
    
    \noindent holds for all $l = 0,...,L$ and all $j \in \mathbb{N}$.
\end{corollary}

\begin{proof}
    The proof is analogously to the proof of lemma \ref{lemma_ReservoirIsSource}. However, in the induction step of the proof of lemma \ref{lemma_ReservoirIsSource}, we make use of the induction hypothesis of lemma \ref{lemma_ConvergenceOfSinglePaths}.
    Now, where we know that lemma \ref{lemma_ConvergenceOfSinglePaths} holds, we do not need to use its induction hypothesis (which bounds its statement by $\hat{\jmath} = l-1 \leq j-1 $), but can use lemma \ref{lemma_ConvergenceOfSinglePaths} itself in the induction step in the proof of this lemma \ref{corollay_ConvergenceOfSinglePaths}. Then, the statement of lemma  \ref{lemma_ReservoirIsSource} holds for all $l = 1, ..., L$ (instead for all $l = 1,..., \min\{j,L\}$).
\end{proof}

\begin{remark}[Intuition of lemma \ref{lemma_ConvergenceOfSinglePaths}]
    Lemma \ref{lemma_ConvergenceOfSinglePaths} or more precisely, corollary \ref{corollay_ConvergenceOfSinglePaths}, gives us the main idea of the physics-informed algorithm (cf. eq. \eqref{align_Algorithm_nested}): Starting from a \textit{reservoir} node $v_r \in V_r$ where the \textit{true} head $h_{v_r}^*$ is known, we iteratively transfer information, i.e., according to the water hydraulics correct head values, to the next node. Hereby, such information reaches a node $v \in V$ with distance $L \in \Le(v)$, i.e., length $L = L(p)$ of some path $p \in \Pa(v_r,v)$ from $v_r$ to $v$, at the $j = L$-th iteration. Therefore, quite intuitively, if we are considering multiple paths and reservoirs, the last iteration where this $v$ can receive new information corresponds to the largest distance $L_{\max}(v)$ from any reservoir $v_r$ to this node $v$. Considering \textit{all} nodes $v \in V$ simultaneously, therefore, the last iteration where \textit{any} $v \in V$ can receive new information corresponds to the largest distance $L_{\max}$ from any reservoir node $v_r \in V_r$ to any node $v \in V$.
    
    However, what we have not considered so far is the fact that paths along which such information can be passed can intersect and hold other reservoir nodes. However, this only leads to the fact that the update rule given in eq. \eqref{align_AlgorithmForSinglePaths_final} changes, but not the iteration at which some node receives new information. We take care of scenarios like these in the next corollary. 
\end{remark}

\begin{corollary}
\label{corollary_UpdatesOfGeneralPaths}
    Let assumption \ref{assumption_Initialization} hold.
    Then for $j \in \mathbb{N}$ and $v \in V$, 
    if $\tilde{h}_v^{(j)} \neq \tilde{h}_v^{(j-1)}$ holds, $j \in \Le(v)$ must hold. 
\end{corollary}

\begin{proof}
    Observing the proof of lemma \ref{lemma_ConvergenceOfSinglePaths} carefully shows that the assumption of the considered path not to be intersected and that the nodes $v_1,...,v_L \in V \setminus V_r$ are no reservoir nodes is only used to prove the explicit form (eq. \eqref{align_AlgorithmForSinglePaths_final}) of $\tilde{h}_{v_l}^{(j)}$ of the node $v_l \in V$ in relation to its neighbor $v_{l-1} \in V$, but is not related to the iteration $j \in \mathbb{N}$:

    More precisely, if the path $p_1 := p = (v_0,...,v_L) := (v_{10}, ..., v_{1l_1}, ..., v_{1L_1})$ is intersected by some other path $p_2 = (v_{20}, ..., v_{2l_2}, ..., v_{2L_2})$ at some node $v = v_{1l_1} = v_{2l_2} \in V$, the paths $\hat{p_1} := (v_{10}, ..., v_{1 l_1})$ and $\hat{p_2} := (v_{20}, ..., v_{2 l_2})$ can be considered as paths in the setting of lemma \ref{lemma_ConvergenceOfSinglePaths}. By this, or more precisely, by corollary \ref{corollay_ConvergenceOfSinglePaths}, $v$ receives information from the reservoir $v_{10}$ at iteration $j = l_1$ and information from the reservoir $v_{20}$ at iteration $j = l_2$. 
    
    \begin{itemize}
        \item If $l_1 = l_2$, at iteration $j = l_1 = l_2$, the node $v$ is updated with the larger of both information coming from the reservoir $v_{10}$ and $v_{20}$, respectively.
        
        \item If $l_1 \neq l_2$, w.l.o.g., we can assume that $l_1 < l_2$ holds. 
        If the information that the node $v$ receives at iteration $l_1$ is larger than the information it receives at iteration $l_2$, it is updated at iteration $l_1$ but not  at iteration $l_2$ anymore. 
        If the information that the node $v$ receives at iteration $l_1$ is smaller than the information it receives at iteration $l_2$, it is updated at iteration $l_1$ \textit{and} $l_2$.
        
        \item Similarly, if $v \in V_r \subset V$, i.e., if $v$ is already a reservoir node, it might be neither updated at iteration $l_1$ \textit{nor} at $l_2$, as the true reservoir head is already larger than the information received by the other reservoirs $v_{10}$ and $v_{20}$.
        
        \item All in all, the node $v$ is only updated at iteration $l_1$ and / or $l_2$, but does not \textit{have to}.
    \end{itemize}
    
    \noindent This scenario can be easily generalized to not only two, but even more paths; more precisely, to all paths $p \in \Pa(v)$ that lead from some reservoir $v_r \in V_r$ to some node $v \in V$. 
    
    To conclude, the value $\tilde{h}_{v}^{(j)}$ is only updated if there exists a path $p \in \Pa(v)$, such that $j = L(p)$ holds. In other words, if $\tilde{h}_{v}^{(j)} \neq \tilde{h}_{v_l}^{(j-1)}$ holds, then $j$ must have been in $j \in \Le(v)$.
\end{proof}

\begin{proof}[Proof of theorem \ref{theorem_ConvergenceOfAlgorithm_appendix}]
    We need to show that $\tilde{h}_v^{(j)}  = \tilde{h}_v^{(L_{\max})}$ holds for all $j \geq L_{\max}$ and all $v \in V$.
    By corollary \ref{corollary_UpdatesOfGeneralPaths}, for any $j \in \mathbb{N}$ and $v \in V$, if $j \not \in \Le(v)$, $\tilde{h}_v^{(j)}  = \tilde{h}_v^{(j-1)}$ holds. By definition of $\Le(v)$ as the set of all path lengths from any $v_r \in V_r$ to $v \in V$ and $L_{\max}$ as the longest path length from any $v_r \in V_r$ to $v \in V$, $j \not \in \Le(v)$ must hold for all $j > L_{\max}$ and all $v \in V$. This proves the claim (starting with $j = L_{\max}$ and then using induction to $j \in \mathbb{N}_{> L_{\max}}$, which we leave as an exercise to the reader).
\end{proof}

\subsection{Proof of Theorem 3.3}
\label{subsection_ProofTheorem3.3}

\begin{theorem}
\label{theorem_MotivationAlgorithm_appendix}
    If in the setting of theorem \ref{theorem_ConvergenceOfAlgorithm_appendix}, $V_r$ corresponds to the reservoirs with known heads $(h_v^*)_{v \in V_r}$ and $\mathbf{\hat{q}} = \mathbf{q}^*$ corresponds to the \textit{true} flows, then $\mathbf{\tilde{q}} = \mathbf{\hat{q}} + \zeta$ holds.
\end{theorem}

\noindent Again, the proof of theorem \ref{theorem_MotivationAlgorithm}/\ref{theorem_ConvergenceOfAlgorithm_appendix} will be a consequence of different lemmata. As a first step, we collect observations about a WDS with true flows $\mathbf{q}^*$.

\begin{lemma}
\label{lemma_Inflows_1}
    For the true flows $\mathbf{q}^*$ and for all $v \in V \setminus V_r$, there exists a $u \in \N(v)$, such that $\sgn(q_{e_{vu}}^*) = -1$ holds.
\end{lemma}

\begin{proof}
    For $v \in V \setminus V_r$, let us assume that for all $u \in N(v)$, $\sgn(q_{e_{vu}}^*) = 1$ holds (i.e., we assume that there only outflows water from $v$). As discussed in section \ref{subsection_WaterHydraulics}, for the node $v \in V \setminus V_r$, its true demand $d_v^*$ (which is not known for reservoirs) satisfies $\sgn(d_v^*) = 1$. By eq. \eqref{align_MassBalance}, however,
    \begin{align*}
        -1
        %=
        %- \sgn(d_v^*)
        =
        \sgn(- d_v^*)
        =
        \sgn\left(
        \textstyle 
        \sum_{u \in \mathcal{N}(v)} q_{e_{vu}}^*
        \right)
        =
        1
    \end{align*}
    
    \noindent holds, which is a contradiction. 
\end{proof}

\noindent While lemma \ref{lemma_Inflows_1} states that for true flows, water inflows to each non-reservoir node $v \in V \setminus V_r$, the next Lemma \ref{lemma_Inflows_2} states that for true flows, this water is supplied by a reservoir $v_r \in V_r$.

\begin{lemma}
\label{lemma_Inflows_2}
    For the true flows $\mathbf{q}^*$ and for all $v \in V \setminus V_r$ there exists a path $p = (v_0, ..., v_L) \in \Pa(v)$ from some $v_r \in V_r$ to $v$, such that 
    \begin{align}
        \label{align_Inflows_2}
        \sgn(q_{e_{v_{l+1} v_l}}^*) = -1
    \end{align}
    
    \noindent holds for all $l = 0,...,L-1$.
\end{lemma}

\begin{proof}
    For each $v \in V \setminus V_r$, we iteratively apply lemma \ref{lemma_Inflows_1} until we reach a reservoir node: 
    For $v_L := v \in V \setminus V_r$, 
    by lemma \ref{lemma_Inflows_1}, 
    there exists a $v_{L-1} \in \N(v_L)$, 
    such that $\sgn(q_{e_{v_L v_{L-1}}}^*) = -1$ holds. 
    Iteratively, 
    for $v_{L - l} \in V \setminus V_r$, 
    by lemma \ref{lemma_Inflows_1}, 
    there exists a $v_{L-l-1} \in \N(v_{L-l})$, 
    such that $\sgn(q_{e_{v_{L - l} v_{L- l - 1}}}^*) = -1$ holds. 
    If for some $\hat{l} \geq 1$, 
    $v_{L - \hat{l}} := v_0 \in V_r$, we stop. 
    Then $p := (v_0, ..., v_L)$ is the path we are looking for.
\end{proof}

\noindent Motivated by lemma \ref{lemma_Inflows_2} and for later purpose, we extend definition \ref{definition_Paths} from section \ref{subsection_ProofTheorem3.2} and also define the neighborhoods of out- and inflows:

\begin{definition}(Water Paths)
\label{definition_Paths_2}
    We denote
    \begin{itemize}
        \item the length of the shortest path from any $v_r \in V_r$ to $v \in V$ by $L_{\min}(v) := \min_{p \in \Pa(v)} L(p)$ and
        
        \item the length of the shortest path from any $v_r \in V_r$ to $v \in V$ that fulfills eq. \eqref{align_Inflows_2} by\footnote{
            As the paths considered here do not exists for nodes $v_r \in V_r$ by definition of a reservoir as the water source (i.e., there is no inflow to $v_r$), we define $L(v_r) := 0$.
        }
        \begin{align*}
            L(v) := \min_{p \in \Pa(v), \text{eq. \eqref{align_Inflows_2} holds}} L(p).
        \end{align*}
    \end{itemize}
\end{definition}

\begin{definition}(Neighborhoods)
\label{definition_Neighborhoods}
    For $v \in V$, we denote the neighborhoods of out- and inflows by
    \begin{align*}
        \N_\pm(v) := \{ u \in \N(v) ~|~ \sgn(q_{e_{vu}}) = \pm 1 \}.
    \end{align*}
\end{definition}

\noindent The latter definition can help us to express the relation between true heads and true flows using the terms 
\begin{align*}
    w_e^* := \mathrm{ReLU}\left(- r_e  \sgn(q_e^*) |q_e^*|^x \right)
\end{align*}

\noindent based on the true flows $\mathbf{q}^*$ for all $e \in E$, analogously to the terms $w_e$ based on the flows $\mathbf{\hat{q}}$ for $e \in E$ and which we know from our algorithm (cf. eq. \eqref{align_Algorithm_nested} and section \ref{subsection_ProofTheorem3.2}):

\begin{lemma}
\label{lemma_Relation_TrueHeadsAndFlows_Algorithm}
    For the true heads $\mathbf{h}^*$ and the true flows $\mathbf{q}^*$, for all $v \in V$,
    \begin{align*}
        &~h_v^*
        \\
        =&~
        h_u^*
        + r_{e_{vu}}  \sgn(q_{e_{vu}}^*) |q_{e_{vu}}^*|^x
        \\
        =&~
        \begin{cases}
        h_u^*
        - \mathrm{ReLU}( - r_{e_{vu}}  \sgn(q_{e_{vu}}^*) |q_{e_{vu}}^*|^x)
        & \text{if } u \in \N_-(v) \\
        h_u^*
        + \mathrm{ReLU}( \text{\textcolor{white}{$-$}} r_{e_{vu}}  \sgn(q_{e_{vu}}^*) |q_{e_{vu}}^*|^x)
        & \text{if } u \in \N_+(v)
        \end{cases}
        \\
        &~
        \begin{cases}
        = 
        h_u^*
        - w_{e_{vu}}^*
        & \text{if } u \in \N_-(v) \\
        \geq
        h_u^*
        - w_{e_{vu}}^*
        & \text{if } u \in \N_+(v)
        \end{cases}
    \end{align*}
    
    \noindent holds for all $u \in \N(v)$.\footnote{
        We do not assume the case $\sgn({\hat{q}}_{e_{vu}}) = 0$ as this would only be the case if there was no flow between two nodes, which in practise will not be the case.} 
\end{lemma}

\begin{proof}
    This is a direct consequence of eq. \eqref{align_HeadLoss}, the definition of the $\mathrm{ReLU}$-function, the definition of the neighborhoods of out- and inflows (cf. definition \ref{definition_Neighborhoods}) and the definition of $w_e^*$ for $e \in E$ (cf. above).
\end{proof}

\noindent Finally, after having collected observations about a WDS with the true flows $\mathbf{q}^*$, we have all tools to make statements about the physics-informed algorithm (cf. eq. \eqref{align_Algorithm_nested}) when being applied to the true reservoir heads and true flows:

\begin{lemma}
\label{lemma_TrueHeadViaShortestInflowPath}
    If in the setting of theorem \ref{theorem_ConvergenceOfAlgorithm_appendix}, $V_r$ corresponds to the reservoirs with known heads $(h_v^*)_{v \in V_r}$ and $\mathbf{\hat{q}} = \mathbf{q}^*$ corresponds to the \textit{true} flows, then
    \begin{align}
    \label{align_TrueHeadViaShortestPath}
        \tilde{h}_v^{(L(v))}
        =
        h_v^*
    \end{align}
    
    \noindent holds for all $v \in V$.
\end{lemma}

\begin{proof}
    Let $v \in V$.
    If $v \in V_r$, $\tilde{h}_v^{(L(v))} = \tilde{h}_v^{(0)} = h_v^*$ holds by definition of $L(v)$ (cf. definition \ref{definition_Paths_2}) and assumption \ref{assumption_Initialization}.
    
    If $v \in V \setminus V_r$, we prove lemma \ref{lemma_TrueHeadViaShortestInflowPath} by induction to $L(v) \in \{1,...,L_{\max}\}$.
    \\
    \\\textbf{Induction base:} 
    If $L(v) = 1$, 
    \begin{align*}
        \tilde{h}_v^{(1)} 
        = 
        \max\{ 
        \tilde{h}_v^{(0)}, 
        \max_{u \in \N(v)}
        \tilde{h}_u^{(0)} - w_{e_{vu}} 
        \}
    \end{align*}
    
    \noindent holds by eq. \eqref{align_Algorithm_nested}. By assumption \ref{assumption_Initialization}, $\tilde{h}_v^{(0)} = c$ and
    \begin{align*}
        \tilde{h}_u^{(0)} - w_{e_{vu}}
        =
        \begin{cases}
            h_{u}^* - w_{e_{vu}}
            & \text{ if } u \in V_r \\
            c - w_{e_{vu}}
            & \text{ if } u \in V \setminus V_r \\
        \end{cases}
    \end{align*}
    
    \noindent holds for all $u \in \N(v)$. For all $u \in \N(v) \cap V_r$, we additionally know by water hydraulics, that $q_{e_{u v}}^*$ must be an outflow and thus, $q_{e_{v u}}^* = - q_{e_{u v}}^*$ must be an inflow (cf. section \ref{subsection_WaterHydraulics}).
    Therefore, $\N(v) \cap V_r = \N_-(v) \cap V_r$ holds.
    Consecutively, using
    (\romannumeral 1) remark \ref{remark_cEstimation},
    (\romannumeral 2) $\mathbf{\hat{q}} = \mathbf{q}^*$ and
    (\romannumeral 3) lemma \ref{lemma_Relation_TrueHeadsAndFlows_Algorithm} for $u \in \N(v) \cap V_r = \N_-(v) \cap V_r$,
    we obtain
    \begin{align*}
        c - w_e 
        \overset{\text{(\romannumeral 1)}}{\leq}~
        c
        \overset{\text{(\romannumeral 1)}}{\leq}~
        h_{u}^* - w_{e_{vu}}
        \overset{\text{(\romannumeral 2)}}{=}~
        h_{u}^* - w_{e_{vu}}^*
        \overset{\text{(\romannumeral 3)}}{=}~
        h_{v}^*
    \end{align*}
    
    \noindent for all $u \in \N(v) \cap V_r$ and $e \in E$.
    Finally, as by definition of $L(v)$, there must exists a $u \in \N(v) \cap V_r$, we conclude that
    \begin{align*}
        \tilde{h}_v^{(1)}
        =
        h_{v}^*
    \end{align*}
    
    \noindent holds, which proves the induction base.
    \\
    \\\textbf{Induction hypothesis:}
    We can assume that for all $v \in V \setminus V_r$, for which $L(v) \in \{1,...,L-1\}$ holds, 
    $
    \tilde{h}_v^{(L(v))} = h_v^*
    $ 
    holds.
    \\
    \\\textbf{Induction step:} 
    We need to show that for all $v \in V \setminus V_r$, for which $L(v) = L \in \mathbb{N}_{> 1}$ holds, eq. \eqref{align_TrueHeadViaShortestPath} holds, given that the induction hypothesis holds.
    
    First of all, for such a $v$, we obtain
    \begin{align*}
        \tilde{h}_v^{(L(v))} 
        := 
        \max\{ 
        \tilde{h}_v^{(L(v)-1)}, 
        \max_{u \in \N(v)}
        \tilde{h}_u^{(L(v)-1)} - w_{e_{vu}} 
        \}
    \end{align*}
    
    \noindent by eq. \eqref{align_Algorithm_nested}. For any $u \in \N(v)$, by definition of a neighbor, $L(u) = L(v) \pm 1$ holds. By that, we need to consider different cases:
    %\\
    \\\textit{Case 1:} If $L(u) = L(v) - 1$ and $u \in \N_-(v)$, using
    (\romannumeral 1) the induction hypothesis,
    (\romannumeral 2) $\mathbf{\hat{q}} = \mathbf{q}^*$ and
    (\romannumeral 3) lemma \ref{lemma_Relation_TrueHeadsAndFlows_Algorithm},
    we obtain
    \begin{align*}
        \tilde{h}_u^{(L(v)-1)} - w_{e_{vu}} 
        \overset{\text{(\romannumeral 1)}}{=}~
        h_u^* - w_{e_{vu}}
        \overset{\text{(\romannumeral 2)}}{=}~
        h_u^* - w_{e_{vu}}^*
        \overset{\text{(\romannumeral 3)}}{=}~
        h_v^*.
    \end{align*}

    \noindent \textit{Case 2:} If $L(u) = L(v) - 1$ and $u \not \in \N_-(v)$, using
    the same arguments, we obtain
    \begin{align*}
        \tilde{h}_u^{(L(v)-1)} - w_{e_{vu}} 
        \overset{\text{(\romannumeral 1)}}{=}~
        h_u^* - w_{e_{vu}}
        \overset{\text{(\romannumeral 2)}}{=}~
        h_u^* - w_{e_{vu}}^*
        \overset{\text{(\romannumeral 3)}}{\leq}~
        h_v^*.
    \end{align*}
    
    \noindent \textit{Case 3:} If $L(u) = L(v) + 1$ and $L(u) = L_{\min}(u)$, we obtain $L(v) - 1 = L(u) - 2 \not \in \Le(u)$ by definition of $ L_{\min}(u)$. Therefore, using
    (\romannumeral 1) corollary \ref{corollary_UpdatesOfGeneralPaths} inductively\footnote{
        We leave it as an exercise to the reader to show that $\tilde{h}_u^{(L(v)-1)} = \tilde{h}_u^{(0)}$ holds. As $L(u) = L_{\min}(u) = L(v) + 1 \geq 1 + 1 = 2$, we emphasize that $u \in V \setminus V_r$ must hold, which is why $\tilde{h}_u^{(L(v)-1)} = \tilde{h}_u^{(0)} = c$ holds by assumption \ref{assumption_Initialization}.
    },
    (\romannumeral 2) $\mathbf{\hat{q}} = \mathbf{q}^*$ and
    (\romannumeral 3) remark \ref{remark_cEstimation},
    we obtain
    \begin{align*}
        \tilde{h}_u^{(L(v)-1)} - w_{e_{vu}} 
        \overset{\text{(\romannumeral 1)}}{=}~
        c - w_{e_{vu}} 
        \overset{\text{(\romannumeral 2)}}{=}~
        c - w_{e_{vu}}^*
        \overset{\text{(\romannumeral 3)}}{\leq}~
        h_v^*.
    \end{align*}
    
    \noindent \textit{Case 4:} If $L(u) = L(v) + 1$ and $L(u) > L_{\min}(u)$, then $\tilde{h}_u^{(L(v)-1)}$ might have been updated already, however, by definition of $L(u)$, these updates are based on outflows. Then, similar to case 2, we can show that in this case, by lemma \ref{lemma_Relation_TrueHeadsAndFlows_Algorithm}, we obtain
    \begin{align*}
        \tilde{h}_u^{(L(v)-1)} - w_{e_{vu}} 
        \leq ~
        h_v^*.
    \end{align*}
    
    \noindent An analogous argument shows that $\tilde{h}_v^{(L(v)-1)} \leq h_v^*$ holds. 
    
    Finally, as by lemma \ref{lemma_Inflows_2}, there must exists a $u \in \N(v)$ as in case 1, we conclude that
    \begin{align*}
        \tilde{h}_v^{(L(v))}
        =
        h_{v}^*
    \end{align*}
    
    \noindent holds, which proves the induction step.
\end{proof}

\noindent Lemma \ref{lemma_TrueHeadViaShortestInflowPath} states that for each node $v \in V$, the information about the true head  $h_v^*$ is passed over the shortest inflow path. The consecutive lemma states that the physics-informed algorithm (cf. eq. \eqref{align_Algorithm_nested}) does not further updates the heads:

\begin{lemma}
\label{lemma_TrueHeadRemains}
    If in the setting of theorem \ref{theorem_ConvergenceOfAlgorithm_appendix}, $V_r$ corresponds to the reservoirs with known heads $(h_v^*)_{v \in V_r}$ and $\mathbf{\hat{q}} = \mathbf{q}^*$ corresponds to the \textit{true} flows, then
    \begin{align}
    \label{align_TrueHeadRemains}
        \tilde{h}_v^{(j)}
        =
        h_v^*
    \end{align}
    
    \noindent holds for all $v \in V$ and all $j \geq L(v)$.
\end{lemma}

\begin{proof}
    Let $v \in V$.
    We prove lemma \ref{lemma_TrueHeadRemains} by induction to $j \geq L(v)$.
    \\
    \\\textbf{Induction base:} If $j = L(v)$, $\tilde{h}_v^{(j)} = h_v^*$ holds by lemma \ref{lemma_TrueHeadViaShortestInflowPath}.
    \\
    \\\textbf{Induction hypothesis:} 
    We can assume that for all $\hat{\jmath} \in \{L(v),...,j-1\}$,
    $
        \tilde{h}_v^{(\hat{\jmath})}
        =
        h_v^*
    $
    holds.
    \\
    \\\textbf{Induction step:} 
    We need to show that for $j \in \mathbb{N}_{> L(v)}$, eq. \eqref{align_TrueHeadRemains} holds, given that the induction hypothesis holds.
    
    Using
    (\romannumeral 1) eq. \eqref{align_Algorithm_nested},
    (\romannumeral 2) the induction hypothesis,
    (\romannumeral 3) $\mathbf{\hat{q}} = \mathbf{q}^*$ and
    (\romannumeral 4) lemma \ref{lemma_Relation_TrueHeadsAndFlows_Algorithm},
    \begin{align*}
        \tilde{h}_v^{(j)} 
        \overset{\text{(\romannumeral 1)}}{=}&~
        \max\{ 
        \tilde{h}_v^{(j-1)}, 
        \max_{u \in \N(v)}
        \tilde{h}_u^{(j-1)} - w_{e_{vu}} 
        \}
        \\
        \overset{\text{(\romannumeral 2)}}{=}&~
        \max\{ 
        h_v^*, 
        \max_{u \in \N(v)}
        h_u^* - w_{e_{vu}} 
        \}
        \\
        \overset{\text{(\romannumeral 3)}}{=}&~
        \max\{ 
        h_v^*, 
        \max_{u \in \N(v)}
        h_u^* - w_{e_{vu}}^*
        \}
        \\
        \overset{\text{(\romannumeral 4)}}{=}&~
        \max\{ 
        h_v^*, 
        \max_{u \in \N_-(v)}
        h_u^* - w_{e_{vu}}^*
        \}
        \\
        \overset{\text{(\romannumeral 4)}}{=}&~
        h_v^*
    \end{align*}
    
    \noindent holds, which proves the induction step.
\end{proof}

\begin{corollary}
\label{corollary_TrueHeads}
    If in the setting of theorem \ref{theorem_ConvergenceOfAlgorithm_appendix}, $V_r$ corresponds to the reservoirs with known heads $(h_v^*)_{v \in V_r}$ and $\mathbf{\hat{q}} = \mathbf{q}^*$ corresponds to the \textit{true} flows, then $\mathbf{\tilde{h}} := \mathbf{\tilde{h}}^{(J)} = \mathbf{h}^*$ holds, i.e., the heads constructed by the physics-informed  algorithm correspond to the true heads.
\end{corollary}

\begin{proof}
    By theorem \ref{theorem_ConvergenceOfAlgorithm_appendix}, $J = L_{\max}$ holds.
    As by definition, $J = L_{\max} \geq L(v)$ holds for all $v \in V$, the claim follows immediately by lemma \ref{lemma_TrueHeadRemains}.
\end{proof}

\begin{lemma}
\label{lemma_TrueFlows}
    If $\mathbf{\tilde{h}} := \mathbf{\tilde{h}}^{(J)} = \mathbf{h}^*$ holds and we compute $\mathbf{\tilde{q}}$ according to 
    \begin{align}
    \label{align_ComputeFlows_appendix}
        {\tilde{q}}_e
        :=
        \sgn({\tilde{h}_v} - {\tilde{h}}_u) \cdot 
        (r_e^{-1} |{\tilde{h}_v} - {\tilde{h}}_u|)^{1/x} 
    \end{align}
    
    \noindent for all $v \in V$ and $e = e_{vu} \in E$, then $\mathbf{\tilde{q}} = \mathbf{q}^*$ holds.
\end{lemma}

\begin{proof}
    By eq. \eqref{align_HeadLoss}, given true heads $\mathbf{h}^*$ and true flows $\mathbf{q}^*$,
    for $v \in V$ and $e = e_{vu} \in E$,
    \begin{align*}
        h_v^* - h_u^* = r_e \sgn(q_e^*) |q_e^*|^x
    \end{align*}
    
    \noindent holds. Because $r_e$ is positive by definition and $|q_e^*|^x$ is obviously non-negative, $\sgn(h_v^* - h_u^*) = \sgn(q_e^*)$ and therefore, $\sgn(q_e^*) \cdot \sgn(h_v^* - h_u^*) = 1$ must hold. Thus, we obtain
    \begin{align*}
        \sgn(q_e^*) |q_e^*|^x
        =
        \sgn(q_e^*) \cdot \sgn(h_v^* - h_u^*)
        \cdot r_e^{-1} (h_v^* - h_u^*)
    \end{align*}
    
    \noindent and therefore,
    \begin{align*}
        |q_e^*|^x
        =
        \sgn(h_v^* - h_u^*)
        \cdot r_e^{-1} (h_v^* - h_u^*)
        =
        r_e^{-1} |h_v^* - h_u^*|.
    \end{align*}
    
    \noindent For the latter equation, we are allowed to take roots:
    \begin{align*}
        |q_e^*|
        =
        (r_e^{-1} |h_v^* - h_u^*|)^{1/x}.
    \end{align*}
    
    \noindent Finally, we use the equality of signs again to obtain 
    \begin{align*}
        q_e^*
        =
        \sgn(h_v^* - h_u^*) \cdot (r_e^{-1} |h_v^* - h_u^*|)^{1/x}.
    \end{align*}
    
    \noindent Therefore, if $\mathbf{\tilde{h}} = \mathbf{h}^*$ holds and we compute $\mathbf{\tilde{q}}$ according to eq. \eqref{align_ComputeFlows_appendix} for all $v \in V$ and $e = e_{vu} \in E$, we find that indeed,
    \begin{align*}
        {\tilde{q}}_e
        :=&~
        \sgn({\tilde{h}_v} - {\tilde{h}}_u) \cdot 
        (r_e^{-1} |{\tilde{h}_v} - {\tilde{h}}_u|)^{1/x}
        \\
        =&~
        \sgn(h_v^* - h_u^*) \cdot (r_e^{-1} |h_v^* - h_u^*|)^{1/x}
        \\
        =&~
        q_e^*
    \end{align*}
    
    \noindent holds.
\end{proof}

\begin{proof}[Proof of theorem \ref{theorem_MotivationAlgorithm_appendix}]
    Theorem \ref{theorem_MotivationAlgorithm_appendix} immediately follows by corollary \ref{corollary_TrueHeads} and lemma \ref{lemma_TrueFlows} (mind the difference $\zeta$ of eq. \eqref{align_ComputeFlowsAndDemands} and \eqref{align_ComputeFlows_appendix}).
\end{proof}

\section{Intuition of Physics-Informed Component}
\label{section_IntuitionOfPhysics-InformedComponent}

While the previous section explains the functionality of our physics-informed algorithm (cf. eq. \eqref{align_Algorithm} or \eqref{align_Algorithm_nested}) in whole detail, this chapter is for readers who are not interested into reading all mathematical details of section \ref{section_TheoreticalBackground}. Readers who have read section \ref{section_TheoreticalBackground} already can skip this section, as there are no new arguments introduced.

To understand the intuition of the physics-informed component described in section \ref{subsection_GlobalPhysicsInformedAlgorithm}, we investigate the case where $\mathbf{\tilde{h}}$ corresponds to the true heads and $\mathbf{\hat{q}}$ to the true flows of a WDS. In such case, by eq. \eqref{align_HeadLoss},
\begin{align*}
    &~\tilde{h}_v
    \\
    =&~
    \tilde{h}_u
    + r_{e_{vu}}  \sgn({\hat{q}}_{e_{vu}}) |{\hat{q}}_{e_{vu}}|^x
    \\
    =&~
    \begin{cases}
    \tilde{h}_u
    - \mathrm{ReLU}( - r_{e_{vu}}  \sgn({\hat{q}}_{e_{vu}}) |{\hat{q}}_{e_{vu}}|^x)
    & \text{if } \sgn({\hat{q}}_{e_{vu}}) < 0 \\
    \tilde{h}_u
    + \mathrm{ReLU}( \text{\textcolor{white}{$-$}} r_{e_{vu}}  \sgn({\hat{q}}_{e_{vu}}) |{\hat{q}}_{e_{vu}}|^x)
    & \text{if } \sgn({\hat{q}}_{e_{vu}}) > 0
    \end{cases}
\end{align*}

\noindent holds for all $v \in V$ and all $u \in \N(v)$.\footnote{
    We do not assume the case $\sgn({\hat{q}}_{e_{vu}}) = 0$ as this would only be the case if there was no flow between two nodes, which in practise will not be the case.} 
Even more, in such case, the max aggregation 
\begin{align*}
    \tilde{h}_v
    &=
    \max_{u \in \N(v)}
    \tilde{h}_u
    + r_{e_{vu}}  \sgn({\hat{q}}_{e_{vu}}) |{\hat{q}}_{e_{vu}}|^x
\end{align*}

\noindent is not needed, as all values considered would be equal. Using the neighborhoods
\begin{align*}
    \N_\pm(v) := \{ u \in \N(v) ~|~ \sgn(\hat{q}_{e_{vu}}) = \pm 1 \}
\end{align*}

\noindent of out- and inflows, we can rewrite this aggregation by 

\begin{align}
\label{align_EdgeMessageGeneration_PhysicsInformed_NonApproximated}
\begin{split}
    ~\tilde{h}_v
    =
    \max \{
    &\max_{u \in \N_-(v)}
    \tilde{h}_u
    - \mathrm{ReLU}( - r_{e_{vu}}  \sgn({\hat{q}}_{e_{vu}}) |{\hat{q}}_{e_{vu}}|^x),
    \\
    &\max_{u \in \N_+(v)}
    \tilde{h}_u
    + \mathrm{ReLU}( \text{\textcolor{white}{$-$}} r_{e_{vu}}  \sgn({\hat{q}}_{e_{vu}}) |{\hat{q}}_{e_{vu}}|^x)
    \}.
\end{split}
\end{align}

\noindent As explained in section \ref{subsection_WaterHydraulics}, demands have a positive and inflows have a negative sign. Hence for eq.\ \eqref{align_MassBalance} to hold, at least one of the summands $q_e^*$ for some $e \in E$ has to be negative, which translates to the fact that there must be an inflow per node $v \in V$. Therefore, if $\mathbf{\hat{q}}$ is correct, $\N_-(v)$ is not empty and we pick a $u \in \N_-(v)$ to describe $\tilde{h}_v$ as
\begin{align*}
    \tilde{h}_v
    =
    \tilde{h}_u
    - \mathrm{ReLU}( - r_{e_{vu}}  \sgn({\hat{q}}_{e_{vu}}) |{\hat{q}}_{e_{vu}}|^x).
\end{align*}

\noindent Then, we easily find that $\tilde{h}_v - \tilde{h}_u = r_{e_{vu}}  \sgn({\hat{q}}_{e_{vu}}) |{\hat{q}}_{e_{vu}}|^x$ holds. If we now reconstruct the flow $\mathbf{\tilde{q}}$ from the heads $\mathbf{\tilde{h}}$ as given in equation \eqref{align_ComputeFlowsAndDemands}, we obtain
\begin{align*}
    {\tilde{q}}_{e_{vu}}
    =&~
    \sgn({\tilde{h}_v} - {\tilde{h}}_u) \cdot 
    (r_{e_{vu}}^{-1} |{\tilde{h}_v} - {\tilde{h}}_u|)^{1/x}
    \\
    =&~
    \sgn({\hat{q}}_{e_{vu}}) \cdot
    (r_{e_{vu}}^{-1} r_{e_{vu}} |{\hat{q}}_{e_{vu}}|^x)^{1/x}
    \\
    =&~
    \sgn({\hat{q}}_{e_{vu}}) |{\hat{q}}_{e_{vu}}| 
    =
    \hat{q}_{e_{vu}}
\end{align*}

\noindent for all $v,u \in V$.

Having a look on the iterative algorithm in section \ref{subsection_GlobalPhysicsInformedAlgorithm}, we find that eq. \eqref{align_EdgeMessageGeneration_PhysicsInformed} is an approximation of eq. \eqref{align_EdgeMessageGeneration_PhysicsInformed_NonApproximated}, where the arguments of the max aggregation, i.e., the messages $m_e^{(j)}$ for an $e \in E$, are correctly computed for inflows only. For outflows, instead of correctly adding positive terms of the kind $\mathrm{ReLU}( r_{e_{vu}}  \sgn({\hat{q}}_{e_{vu}}) |{\hat{q}}_{e_{vu}}|^x)$, these are set to zero. Therefore, these messages are smaller than they should be under the physical properties of the WDS. Nevertheless, as these messages are \textit{smaller}, taking the maximum in the iterative scheme of eq. \eqref{align_Algorithm} automatically corrects the error we make for outflows whenever an inflow is observed, which is the case for a correct set of flows $\mathbf{\hat{q}}$.

To conclude, in the case where the estimated flows $\mathbf{\hat{q}}$ are equal to the \textit{true} flows (and thus, the heads $\mathbf{\tilde{h}}$ are equal to the true heads, cf. corollary \ref{corollary_TrueHeads}), the reconstructed flows $\mathbf{\tilde{q}}$ are equal to the flows  $\mathbf{\hat{q}}$ learned by the trainable model $f_1$ (cf. theorem \ref{theorem_MotivationAlgorithm}/\ref{theorem_MotivationAlgorithm_appendix}). 
This motivates to enforce the reconstructed flows $\mathbf{\tilde{q}}$ to be equal to the flows $\mathbf{\hat{q}} = f_1((\mathbf{D}, \mathbf{Q}), \Theta)$ estimated by the first component $f_1$ using a suitable loss function over the learnable parameters $\Theta$. By that, the component $f_2$ acts as a correction to the (possibily incorrect) flows computed by $f_1$ using the water hydraulics from section \ref{subsection_WaterHydraulics}. This guides the overall model $f = f_2 \circ f_1(\cdot, \Theta)$ to converge to a valid set of flows that obey both hydraulic principles (eq. \eqref{align_MassBalance} and eq. \eqref{align_HeadLoss}) after multiple iterations $K \in \mathbb{N}$ (cf. section \ref{subsection_OverallModelandTraining}).
%Only inflows will be retained while outflows will are set to zero due to the usage of the $\mathrm{ReLU}$-function. 

\begin{remark}[Derivation of eq. \eqref{align_ComputeFlowsAndDemands}]
    We compute ${\tilde{q}}_e$ as given in eq. \eqref{align_ComputeFlowsAndDemands}, because by eq. \eqref{align_HeadLoss}, 
    for $v \in V$ and $e = e_{vu} \in E$,
    \begin{align*}
        {\tilde{h}_v} - {\tilde{h}}_u = r_e \sgn({\tilde{q}}_e) |{\tilde{q}}_e|^x,
    \end{align*}
    
    \noindent \textit{should} hold. As $r_e$ is positive by definition and $|{\tilde{q}}_e|^x$ is obviously non-negative, $\sgn({\tilde{h}_v} - {\tilde{h}}_u) = \sgn({\tilde{q}}_e)$ and therefore, $\sgn({\tilde{q}}_e) \cdot \sgn({\tilde{h}_v} - {\tilde{h}}_u) = 1$ must hold. Thus, we obtain
    \begin{align*}
        \sgn({\tilde{q}}_e) |{\tilde{q}}_e|^x
        =
        \sgn({\tilde{q}}_e) \cdot \sgn({\tilde{h}_v} - {\tilde{h}}_u)
        r_e^{-1} ({\tilde{h}_v} - {\tilde{h}}_u)
    \end{align*}
    
    \noindent and therefore,
    \begin{align*}
        |{\tilde{q}}_e|^x
        =
        \sgn({\tilde{h}_v} - {\tilde{h}}_u)
        r_e^{-1} ({\tilde{h}_v} - {\tilde{h}}_u)
        =
        r_e^{-1} |{\tilde{h}_v} - {\tilde{h}}_u|.
    \end{align*}
    
    \noindent For the latter, we are allowed to take roots:
    \begin{align*}
        |{\tilde{q}}_e|
        =
        (r_e^{-1} |{\tilde{h}_v} - {\tilde{h}}_u|)^{1/x}.
    \end{align*}
    
    \noindent Finally, we use the equality of signs again to obtain eq. \eqref{align_ComputeFlowsAndDemands}.
\end{remark}

\noindent Note that this remark explains what we do in the proof of lemma \ref{lemma_TrueFlows}.

% \vskip 100mm
\newpage

\section{Training Algorithm}
\label{app_TrainingAlgorithm}

The complete training algorithm is given in algorithm \ref{alg: training algorithm}.

\begin{algorithm}[!htbp]
\caption{Training}
\label{alg: training algorithm}
\small
\textbf{Inputs}: $V, E, \mathbf{d}^*, \mathbf{h}^*, r$\\
\textbf{Parameter}: $\alpha, \beta, \gamma, \eta, \lambda, B, S, I, J, K, \rho, \delta$\\
\textbf{Output}: $\mathbf{\hat{d}}, \mathbf{\tilde{d}}, \mathbf{\tilde{h}}, \mathbf{\hat{q}}, \mathbf{\tilde{q}}$
\begin{algorithmic}[1] %[1] enables line numbers
\FOR{$b\in B$}
    \FOR{$s\in\{0,...,S_b-1\} $}
        \STATE $\mathbf{h}^{(0)} = (h_v^{(0)})_{v \in V}$
        \STATE $\mathbf{\tilde{h}}^{(0)} = (h_v^{(0)})_{v \in V}$
        \STATE $h_v^{(0)} := \left\{ \begin{aligned}  0 \quad & \text{ if } v \in V \setminus V_r, \\ h_v^* \quad & \text{ if } v \in V_r \end{aligned} \right\} $

        \STATE $\mathbf{D}^{(0)} = (\mathbf{d_1}^{(0)}, \mathbf{d_2}^{(0)})$
        \STATE $\mathbf{Q}^{(0)} = (\mathbf{q_1}^{(0)}, \mathbf{q_2}^{(0)})$
        \STATE $d_{v1}^{(0)}, d_{v2}^{(0)} := \left\{ \begin{aligned} d_v^* \quad & \text{ if } v \in V \setminus V_r, \\ 0 \quad & \text{ if } v \in V_r \end{aligned} \right\}  $
        \STATE $q_{e1}^{(0)}, q_{e2}^{(0)} := r_e^{-1} (h_v^{(0)} - h_u^{(0)})^{\frac{1}{x}}$
        \FOR{$v\in V$}
            \FOR{$k\in\{0,...,K-1\}$}
                
                \STATE $\mathbf{g}_v^{(k)} := \alpha(\mathrm{SeLU}({d}_v^{(k)}))$
                \STATE $\mathbf{z}_e^{(k)} := \beta(\mathrm{SeLU}({q}_e^{(k)}))$
                \FOR{$i\in\{0,...,I-1\}$}
                    \STATE $\mathbf{m}_e^{(k)(i)}
                    := $
                    \STATE $\gamma^{(i)}(\mathrm{SeLU}(\mathbf{g}_u^{(k)(i)} \parallel \mathbf{g}_v^{(k)(i)} \parallel \mathbf{z}_e^{(k)(i)}))$  
                    \STATE $\mathbf{g}_v^{(k)(i+1)} = \eta^{(i)} (\max_{u \in \mathcal{N}(v)} \mathbf{m}_{e_{vu}}^{(k)(i)})$  
                    \STATE $\mathbf{z}_{e_{vu}}^{(k)(i+1)} = \mathbf{m}_{e_{vu}}^{(k)(i)}$
                \ENDFOR
                \STATE $\hat{q}_e^{(k+1)} 
                        := 
                        q_{e1}^{(k)} 
                        + $
                \STATE $\lambda(\mathrm{SeLU}(\mathbf{g}_u^{(k)(I)} \parallel \mathbf{g}_v^{(k)(I)} \parallel \mathbf{z}_e^{(k)(I)})$
                       
                \STATE ${\hat{q}}_{e_{vu}}^{(k+1)} = (({\hat{q}}_{e_{vu}}^{(k+1)})_{in} \parallel -({\hat{q}}_{e_{vu}}^{(k+1)})_{in})$ 
                \STATE ${\hat{d_v}}^{(k+1)} = -\sum_{u \in \mathcal{N}(v)} {\hat{q}}_{e_{vu}}^{(k+1)}$
                \FOR{$j\in\{0,...,J-1\}$}
                    \STATE ${m}_e^{(k)(j)}
                    := $
                    \STATE ${\tilde{h}}_u^{(k)(j)}
                    -
                    \mathrm{ReLU}\left(- r_e  \sgn({\hat{q}}_e^{(k+1)}) |{\hat{q}}_e^{(k+1)}|^x \right)$ 
                    \STATE ${m}_v^{(k)(j)} 
                    := 
                    \max_{u \in \N(v)} {m}_{e_{vu}}^{(k)(j)}$
                    \STATE ${\tilde{h}}_v^{(k)(j+1)} 
                    := 
                    \max\{ {\tilde{h}}_v^{(k)(j)}, {m}_v^{(k)(j)} \}$
                \ENDFOR
                \STATE $\tilde{h}^{(k+1)} := \tilde{h}^{(k)(J)}$
                \STATE ${\tilde{q}}_e^{(k+1)}
                := $
                \tiny
                \STATE $ \sgn({\tilde{h}_v}^{(k+1)} - {\tilde{h}}_u^{(k+1)}) \cdot 
                (r_e^{-1} |{\tilde{h}_v^{(k+1)}} - {\tilde{h}}_u^{(k+1)}|)^{1/x} + \zeta$
                \small
                \STATE ${\tilde{d}}_v^{(k+1)} 
                := - 
                \textstyle 
                \sum_{u \in \mathcal{N}(v)} {\tilde{q}}_e^{(k+1)}$
            \ENDFOR
        \ENDFOR
    \ENDFOR
    % \STATE $\mathcal{L} ({d}, {\hat{d}}) = \frac{1}{S\cdot N} \sum_{v=1}^{S \cdot N} | {d}_{v} - {\hat{d}}_{v} |$
    % \STATE $\mathcal{L} ({d}, {\tilde{d}}) = \frac{1}{S\cdot N} \sum_{v=1}^{S \cdot N} | {d}_{v} - {\tilde{d}}_{v} | $
    % \STATE $\mathcal{L} ({\hat{q}}, {\tilde{q}}) = \frac{1}{S\cdot M} \sum_{e=1}^{S \cdot M} | {\hat{q}}_{e} - {\tilde{q}}_{e} |$
    \STATE Minimize 
    \STATE $\mathcal{L} = 
    \mathcal{L} (\mathbf{d}^{(K)}, \mathbf{\hat{d}}^{(K)}) 
    + \rho
    \mathcal{L} (\mathbf{d}^{(K)}, \mathbf{\tilde{d}}^{(K)}) 
    + \delta
    \mathcal{L} (\mathbf{\hat{q}}^{(K)}, \mathbf{\tilde{q}}^{(K)})$
\ENDFOR
\STATE \textbf{return} $\mathbf{\hat{d}}, \mathbf{\tilde{d}}, \mathbf{\tilde{h}}, \mathbf{\hat{q}}, \mathbf{\tilde{q}}$
\end{algorithmic}
\end{algorithm}

\pagebreak

\section{Water Distribution Systems}
\label{app_wds}

We use five popular WDS for our experiments. These WDS have different layouts. Four of these have a single reservoir, while the Pescara WDS has three reservoirs. These are exhibited in fig. \ref{fig: hanoi}, \ref{fig: fossolo}, \ref{fig: pescara}, \ref{fig: l_town}, \ref{fig: zhijiang}.

\begin{figure}[!htbp]
\centering
\resizebox{.99\columnwidth}{!}{
\includegraphics[]{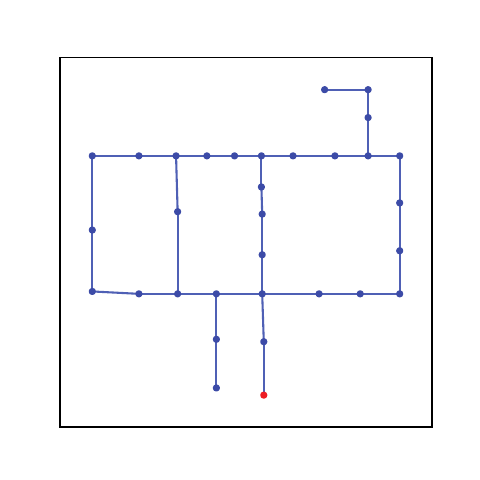} 
}
\caption{Hanoi WDS  \cite{vrachimis2018leakdb}. Nodes in blue are consumers while the node in red is the reservoir.}
\label{fig: hanoi}
\end{figure}

\begin{figure}[!htbp]
\centering
\resizebox{.99\columnwidth}{!}{
\includegraphics[]{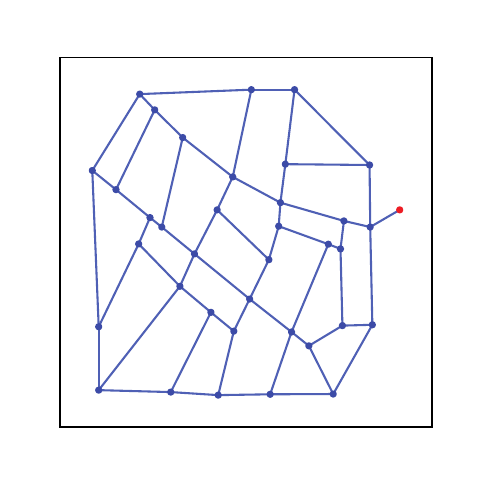} 
}
\caption{Fossolo WDS \cite{zhijiang}. Nodes in blue are consumers while the node in red is the reservoir.}
\label{fig: fossolo}
\end{figure}

\begin{figure}[!htbp]
\centering
\resizebox{.99\columnwidth}{!}{
\includegraphics[]{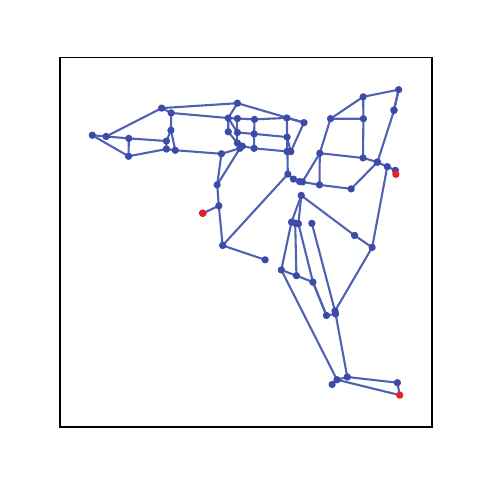} 
}
\caption{Pescara WDS \cite{zhijiang}. Nodes in blue are consumers while the node in red is the reservoir.}
\label{fig: pescara}
\end{figure}

\begin{figure}[!htbp]
\centering
\resizebox{.99\columnwidth}{!}{
\includegraphics[]{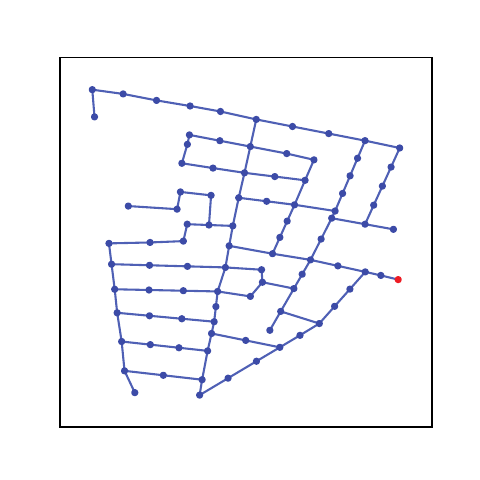} 
}
\caption{L-TOWN Area-C WDS  \cite{vrachimis2020battledim}. Nodes in blue are consumers while the node in red is the reservoir.}
\label{fig: l_town}
\end{figure}

\begin{figure}[t]
\centering
\resizebox{.99\columnwidth}{!}{
\includegraphics[]{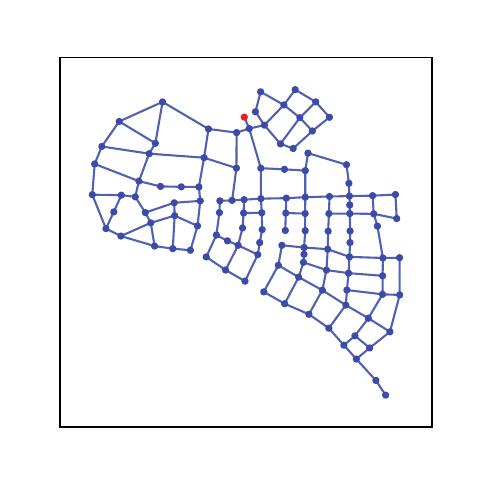} 
}
\caption{Zhi Jiang WDS \cite{zhijiang}. Nodes in blue are consumers while the node in red is the reservoir.}
\label{fig: zhijiang}
\end{figure}

\newpage

\end{document}